\documentclass[twoside,11pt,letterpaper]{article}

\usepackage{jmlr2e}

\jmlrheading{13}{2012}{1-30}{2/11; Revised 1/12}{XX/12}{Chunhua Shen, Junae Kim, Lei Wang, and Anton van den Hengel}
\ShortHeadings{Metric Learning Using Boosting-like Algorithms}
{Shen, Kim, Wang and van den Hengel}

\firstpageno{1}

\usepackage{pslatex}

\usepackage{graphicx,color,url,xspace}

\graphicspath{{./}{./Fig/}{./Figs/}}

\usepackage{amssymb}
\usepackage{amsthm}
\usepackage{amsmath,amscd}
\renewcommand{\boldsymbol}[1]{\pmb{#1}}

\usepackage{verbatim}

\usepackage{multirow}
\usepackage{rotating}
\usepackage{pdflscape}

\DeclareMathAlphabet{\mathcal}{OMS}{cmsy}{m}{n}

\usepackage{algorithm2e}
\let\savedalgorithm\algorithm
\let\savedendalgorithm\endalgorithm
\newenvironment{algorithmic}{%
\savedalgorithm
}{%
\savedendalgorithm
}

\usepackage{algorithm}

\theoremstyle{plain}
\newtheorem{theorem}{Theorem}[section]

\newtheorem{lemma}{Lemma}[section]

\theoremstyle{definition}
\newtheorem{definition}{Definition}[section]

\theoremstyle{remark}

\providecommand{\newoperator}[3]{%
    \newcommand*{#1}{\mathop{#2}#3}}

\def\eg{\emph{e.g.}\xspace}
\def\ie{\emph{i.e.}\xspace}

\def\wrt{{w.r.t.}\xspace}

\def\aka{\emph{a.k.a.,}\xspace}

\def\vs{{versus}\xspace}

\def\BoostMetric{{\sc BoostMetric}\xspace}

\let\forany\forall

\def\T{{\!\top}}

\def\Real{\mathbb{R}}

\def\rank{\operatorname{\bf   Rank}}
\def\trace{\operatorname{\bf  Tr}}

\def\convhull{\operatorname{\bf   Conv}}
\def\logit{{\rm logit}}

\newoperator{\sst}{\mathrm{s.t.\!\!:}}{\nolimits}

\newoperator{\argmin}{\mathrm{argmin}}{\limits}
\newoperator{\argmax}{\mathrm{argmax}}{\limits}

\def\innerp#1#2{{\left<#1, #2 \right>}}
\def\fnorm#1#2{\left\| #2 \right\|_{ \mathrm{#1} } }

\def\psd{\succcurlyeq}

\def\btheta{\boldsymbol w}

\def\SS{{\mathcal I}}

\def\dist{{\bf dist}}

\def\PSD{p.s.d.\@\xspace}
\def\tsum{{\textstyle \sum}}

\def\A{{\mathbf A}}
\def\Z{{\mathbf Z}}
\def\H{{\mathbf H}}
\def\X{{\mathbf X}}
\def\L{{\mathbf L}}
\def\U{{\mathbf U}}

\def\ba{{\mathbf a}}
\def\bx{{\boldsymbol x}}
\def\bw{{\boldsymbol w}}
\def\bu{{\boldsymbol u}}
\def\bv{{\boldsymbol v}}
\def\bp{{\boldsymbol p}}

\def\blambda{{\boldsymbol \lambda}}

\def\I{{\mathbf I}}

\def\bbS{{\mathbb S}}
\def\bbZ{{\mathbb Z}}

\def\Sp{\mathrm{Sp}}

\def\eigenmax{{\lambda_{\mathrm{max}}}}

\def\Dot{{\color{red}{\pmb \cdot}}}

\def\cite{\citep}

\begin{document}
\title{Positive Semidefinite Metric Learning \\  Using Boosting-like Algorithms}

\author{\name Chunhua Shen \email chunhua.shen@adelaide.edu.au \\
                   \addr  The University of Adelaide, Adelaide, SA 5005, Australia
\AND
\name Junae Kim \email junae.kim@nicta.com.au \\
                   \addr  NICTA, Canberra Research Laboratory,
                          Locked Bag 8001, Canberra, ACT 2601, Australia
\AND
\name Lei Wang  \email leiw@uow.edu.au  \\
                \addr  University of Wollongong, Wollongong, NSW 2522, Australia
\AND
\name Anton van den Hengel \email anton.vandenhengel@adelaide.edu.au \\
                 \addr The University of Adelaide, Adelaide, SA 5005,  Australia
}

\editor{S\"oren Sonnenburg, Francis Bach, Cheng Soon Ong}
\maketitle

\begin{abstract}%
    The success of many machine learning and pattern recognition methods relies
    heavily upon the identification of
    an appropriate distance metric on the input data.
    It is often beneficial to learn such a metric from the input training data,
    instead of using a default one such as the Euclidean distance.
    In this work, we propose a boosting-based technique,
    termed \BoostMetric, for
    learning a quadratic Mahalanobis distance metric.
    Learning a valid Mahalanobis distance metric requires enforcing the
    constraint that the matrix parameter to the metric remains positive
    semidefinite.
    Semidefinite programming is often used to enforce this constraint,
    but does not scale well and is not easy to implement.
    \BoostMetric is instead based on the observation that
    any positive semidefinite matrix can be
    decomposed into a linear combination of
    trace-one rank-one matrices.  \BoostMetric thus
    uses rank-one positive semidefinite matrices
    as weak learners within an efficient and scalable
    boosting-based learning process.
    The resulting methods are easy to implement,
    efficient, and can accommodate various
    types of constraints.
    We extend traditional boosting algorithms in that its weak learner is a positive
    semidefinite matrix with trace and rank being one rather than a classifier
    or regressor.
    Experiments on various datasets demonstrate that the proposed algorithms compare
    favorably to those state-of-the-art methods
    in terms of classification accuracy and running time.

\end{abstract}

\begin{keywords}
      Mahalanobis distance, semidefinite programming, column generation,
      boosting, Lagrange duality, large margin nearest neighbor.
\end{keywords}

\section{Introduction}
\label{sec:intro}

The identification of an effective metric by which to measure distances between data points is an
essential component of many machine learning algorithms including $k$-nearest neighbor ($ k $NN),
$k$-means clustering, and kernel regression. These methods have been applied to a range of problems,
including image classification and retrieval
\cite{Hastie1996Adaptive,Yu2008Distance,Jian2007Metric,
Xing2002Distance,Bar2005Mahalanobis,Boiman2008,Frome2007} amongst a host of others.

The Euclidean distance has been shown to be effective in a wide variety of circumstances.  
\citet{Boiman2008}, for instance, showed that in generic object recognition with local
features,  $ k $NN with a Euclidean metric can achieve comparable or better accuracy  than more
sophisticated classifiers such as support vector machines (SVMs).  The Mahalanobis distance
represents a generalization of the Euclidean distance, and offers the opportunity to learn a
distance metric directly from the data.  This learned Mahalanobis distance approach has been shown
to offer improved performance over Euclidean distance-based approaches, and was particularly shown
by \citet{Wang2010Image} to represent an improvement upon the method of 
\citet{Boiman2008}.  It is the prospect of a significant performance improvement from fundamental
machine learning algorithms which inspires the approach presented here.

    If we let $ \ba_i, i=1,2\cdots,$ represent a set of points in $\mathbb R^D$,
then the Mahalanobis
distance, or Gaussian quadratic distance, between two points is
 \begin{equation}
 \Vert{\ba_i - \ba_j } \Vert_{\X}
 = \sqrt{( \ba_i - \ba_j  ) ^\T \X
 ( \ba_i - \ba_j  )},
 \end{equation}
where $ \X \psd 0$ is a positive semidefinite (\PSD) matrix.  The Mahalanobis
distance is thus parameterized by a \PSD matrix, and methods for learning
Mahalanobis distances are therefore often framed as constrained semidefinite
programs.  The approach we propose here, however, is based on boosting, which is
more typically used for learning classifiers.  The primary motivation for the
boosting-based approach is that it scales well, but its efficiency in dealing
with large data sets is also advantageous.  The learning of Mahalanobis distance
metrics represents a specific application of a more general method for matrix
learning which we present below.

We are interested here in the case where the training data consist of a set of
constraints upon the relative distances between data points,
            \begin{equation}
            \label{eq:SSDefn}
            \SS = \{(\ba_i, \ba_j, \ba_k) \, | \, \dist_{ij} < \dist_{ik} \},
            \end{equation}
where $\dist_{ij}$ measures the distance between $\ba_i$ and $\ba_j$. Each such constraint implies
that ``$ \ba_i$ is closer to $ \ba_j $ than $ \ba_i$ is to $ \ba_k$''.  Constraints such as these
often arise when it is known that $\ba_i$ and $\ba_j$ belong to the same class of data points while
$\ba_i, \ba_k$
            belong to different classes.
            These comparison constraints are thus often much easier to obtain
            than either the class labels or distances between data
            elements   \cite{Schultz2004Learning}. For example, in video content
            retrieval, faces extracted from successive frames at close locations
            can be safely assumed to belong to the same person, without
            requiring the individual to be identified.  In web search, the
            results returned by a search engine are ranked according to the
            relevance, an ordering which allows a natural conversion into a set
            of constraints.

            The problem of learning a \PSD matrix such as $\X$ can be formulated in terms of
            estimating a projection matrix $\L$ where $ \X = \L  \L^\T $.  This approach has the
            advantage that the \PSD constraint is enforced through the parameterization, but the
            disadvantage is that the relationship between the distance
            measure and the parameter matrix is less direct.
            In practice this approach has lead to local, rather than globally optimal
            solutions, however (see \cite{Goldberger2004Neighbourhood} for
            example).
                        
            Methods such as
            \cite{Xing2002Distance,Weinberger05Distance,Weinberger2006Unsupervised,Globerson2005Metric}
            which seek $\X$ directly are able to guarantee global optimality, but at the cost of a
            heavy computational burden and poor scalability as
            it is not trivial to preserve the semi\-de\-fi\-ni\-te\-ness
            of $ \X $ during the course of learning.
            Standard approaches such as  interior-point (IP) Newton methods need to calculate the
            Hessian.  This typically requires $O(D^4)$ storage and has worst-case computational
            complexity of approximately $ O( D^{6.5}) $ where $D$ is the size of the \PSD matrix.
            This is
            prohibitive for many real-world problems.
            An alternating projected (sub-)gradient approach is adopted in
            \cite{Weinberger05Distance,Xing2002Distance,Globerson2005Metric}.
            The disadvantages of this algorithm, however, are:
            1)  it is not easy to implement;
            2)  many parameters are involved;
            3)  usually it converges slowly.

            We propose here a method for learning a \PSD matrix labeled \BoostMetric.  
            The method is
            based on the observation that
            any positive semidefinite matrix can
            be decomposed into a linear positive
            combination of trace-one rank-one matrices.
            The weak learner in \BoostMetric is
            thus
            a trace-one rank-one \PSD matrix.
            The proposed \BoostMetric algorithm has the following desirable
            properties:
            \begin{enumerate}
            \item
            \BoostMetric is efficient and scalable.
            Unlike most existing methods,
            no semidefinite programming is
            required. At each iteration, only the largest eigenvalue
            and its corresponding eigenvector are needed.

            \item
            \BoostMetric can accommodate various types of constraints.
            We demonstrate the use of the method to learn a
            Mahalanobis distance on the basis of a set of proximity
            comparison constraints.

            \item
            Like AdaBoost, \BoostMetric does not have any parameter to tune.
            The user only needs to know when to stop.
            Also like AdaBoost it is easy to implement. No sophisticated
            optimization techniques are involved.
            The efficacy and efficiency of the proposed \BoostMetric is demonstrated
            on various datasets.

            \item
            We also propose a totally-corrective version of \BoostMetric.
            As in TotalBoost \cite{Warmuth2006Total} the weights of all the
            selected weak learners (rank-one matrices) are updated at each iteration.

            Both the stage-wise \BoostMetric and totally-corrective \BoostMetric methods
            are very easy to implement.
            \end{enumerate}

            The primary contributions of this work are therefore as
            follows:
             1)   We extend traditional boosting algorithms
             such that each weak learner is a
                matrix with the trace and rank of one---which must be positive semidefinite---rather
                than a classifier or regressor;
             2)   The proposed algorithm
             can be used to solve many
                semidefinite optimization problems in machine learning and computer vision.
                We demonstrate the scalability and effectiveness of our algorithms on metric
                learning.
            Part of this work appeared
            in \citet{Shen2008PSD,Shen2009SDP}.
            More theoretical analysis and experiments are included in this version.
            Next, we review some relevant work before we present our algorithms.

            \subsection{Related Work}

            Distance metric learning is closely related to subspace methods.
            Principal component analysis (PCA) and linear discriminant analysis
            (LDA) are two classical dimensionality reduction techniques.
            PCA finds the subspace that
            captures the maximum variance
            within
            the input data
            while LDA tries to identify the projection which maximizes the 
            between-class distance and minimizes the within-class variance.
            Locality preserving projection (LPP) finds a
            linear projection that preserves the neighborhood structure of the data set
            \cite{LPPface2005}.
            Essentially, LPP linearly approximates
            the eigenfunctions of the Laplace Beltrami operator on the underlying
            manifold. The connection between LPP and LDA is also revealed in
            \cite{LPPface2005}.
            \citet{Wang2010DLE} extended LPP to supervised multi-label
            classification.
            Relevant component analysis (RCA) \cite{Bar2005Mahalanobis}
            learns a metric from {\em equivalence} constraints.
            RCA can be viewed as
            extending
            LDA by
            incorporating must-link constraints and cannot-link constraints
            into the learning procedure.
            Each of these methods may be seen as devising a  linear projection
            from the input space to a lower-dimensional output space.
            If this projection is characterized by the matrix $\L$, then note that
            these methods may be related to the problem of interest
            here by observing 
                                   $ \X = \L  \L^\T $.
            This typically implies that $\X$ is rank-deficient.

            Recently, there
            has been significant
            research interest in
            supervised distance metric learning using
            side information that is typically presented
            in a set of pairwise constraints.
            Most of these methods,
            although
            appearing in different formats,
            share a similar essential idea:
            to learn an optimal distance metric by keeping training examples
            in equivalence constraints close, and at the same time,
            examples in in-equivalence
            constraints well separated.
            Previous work of \cite{Xing2002Distance,Weinberger05Distance,
            Jian2007Metric,Goldberger2004Neighbourhood,Bar2005Mahalanobis,
            Schultz2004Learning} fall into this category.
            The requirement
            that $\X$ must be
            \PSD has led to the
            development of a number of methods
            for learning a Mahalanobis distance
            which rely upon constrained semidefinite programing.
            This approach has a number of limitations, however,
            which we now discuss with reference to the problem of learning a
            \PSD matrix from a set of constraints upon pairwise-distance comparisons.
            Relevant work on this topic includes
            \cite{Bar2005Mahalanobis,Xing2002Distance,
            Jian2007Metric,Goldberger2004Neighbourhood,
            Weinberger05Distance,Globerson2005Metric} amongst others.

            \citet{Xing2002Distance}
            first proposed the idea of learning a
            Mahalanobis metric for clustering using convex optimization.
            The inputs are two sets: a similarity set and  a dis-similarity set.
            The algorithm maximizes the distance between points in the dis-similarity
            set under the constraint that the distance between points in the similarity
            set is upper-bounded.
            Neighborhood component analysis (NCA) \cite{Goldberger2004Neighbourhood}
            and large margin nearest
            neighbor (LMNN) \cite{Weinberger05Distance} learn a metric by maintaining
            consistency in data's neighborhood and keep a large margin
            at the boundaries of different classes.
            It has been shown in \cite{Weinberger2009Distance,Weinberger05Distance}
            that LMNN delivers the state-of-the-art
            performance among most distance metric learning algorithms.
            Information theoretic metric learning (ITML) learns a suitable metric
            based on information theoretics \cite{Davis2007Info}. To partially alleviate the
            heavy computation of standard IP Newton methods, Bregman's cyclic projection
            is used in \citet{Davis2007Info}. This idea is extended in \citet{Wang2009Info},
            which has a closed-form solution and is computationally efficient.

            There have been a number of approaches developed which aim to improve the 
            scalability of the process of learning a metric parameterized by a \PSD metric $ \X $.
            For example, \citet{Rosales2006} approximate the \PSD cone
            using a set of linear constraints based on the diagonal dominance theorem.
            The approximation is not accurate, however, in the sense that it imposes too strong
            a condition on the learned matrix---one may not want to learn a diagonally dominant
            matrix. Alternative optimization is used in \cite{Xing2002Distance,Weinberger05Distance}
            to solve the semidefinite problem iteratively. At each iteration,
            a full eigen-decomposition is applied to project the solution back onto
            the \PSD cone.
            \BoostMetric is conceptually very different to this
                        approach, and additionally only requires the calculation
                        of the first eigenvector.
             \citet{Tsuda2005} proposed to use 
             matrix logarithms and exponentials to preserve positive
             definiteness. For the application of semidefinite kernel
             learning, they designed a matrix exponentiated gradient
             method to optimize  von Neumann divergence based    
             objective functions. 
             At each iteration of matrix exponentiated gradient, a full
             eigen-decomposition is needed. In contrast, we only need
             to find the leading eigenvector.

            The approach proposed here is directly inspired by the LMNN proposed
            in~\cite{Weinberger2009Distance,Weinberger05Distance}.
            Instead of using the hinge loss,
            however,
            we use the exponential loss and logistic loss functions in order to
            derive an AdaBoost-like (or LogitBoost-like) optimization procedure.
            In theory, any differentiable convex loss function can be applied here.
            Hence, despite similar
            purposes, our algorithm differs essentially in the optimization.
            While the formulation of LMNN looks more similar to
            SVMs, our algorithm, termed \BoostMetric, largely draws upon
            AdaBoost \cite{Schapire1999Boosting}.

            Column generation  was first proposed by
            \citet{Dantzig1960CG} for solving a particular form of
            structured linear program with an extremely large number
            of variables.  The general idea of column generation is
            that, instead of solving the original large-scale problem
            (master problem), one works on a restricted master problem
            with a reasonably small subset of the variables at each
            step. The dual of the restricted master problem is solved
            by the simplex method, and the optimal dual solution is
            used to find the new column to be included into the
            restricted master problem.  LPBoost
            \cite{Demiriz2002LPBoost} is a direct application of
            column generation in boosting.  Significantly, LPBoost
            showed that in an LP framework, unknown weak hypotheses
            can be learned from the dual although the space of all
            weak hypotheses is infinitely large.  \citet{Dual2010Shen}
            applied column generation to boosting with general loss
            functions.  It is these results that underpin
            \BoostMetric.

        The remaining content is organized as follows. In Section \ref{sec:preliminaries}
        we present some preliminary mathematics. In Section \ref{sec:algorithm},
        we show the main results.
        Experimental results are provided in Section \ref{sec:exp}.

\section{Preliminaries}
        \label{sec:preliminaries}

        We introduce some fundamental concepts that are necessary
        for setting up our problem. First, the notation used in this paper is as follows.

        \subsection{Notation}

        Throughout this paper, a matrix is denoted by a bold upper-case
        letter ($\X$); a column vector is denoted by a bold lower-case
        letter ($ \bx $).
        The $ i$th row of $\X $ is denoted by $ \X_{i:} $ and the $
        i$th column $ \X_{:i}$.
        $ \boldsymbol 1 $ and
        $ \boldsymbol 0 $ are column vectors of $ 1$'s and $ 0$'s,
        respectively. Their size should be clear from the context.
        We denote the space of $ D \times D $ symmetric matrices by $
        \bbS^D$, and positive semidefinite matrices by $ \bbS^D_+$. $
        \trace(\cdot) $ is the trace of a symmetric matrix and
            $
            \innerp{\X}{\Z} = \trace(\X\Z^\T) = \sum_{ij}\X_{ij}\Z_{ij}
            $
        calculates the inner product of two matrices.  An element-wise
        inequality between two vectors like $ \bu \leq \bv $ means $
        u_i \leq v_i $ for all $ i $.
        We use $ \X \psd 0 $
        to indicate that matrix $ \X $ is positive semidefinite.
        For a matrix $ \X  \in \bbS^D$,
        the following statements are equivalent:
        1) $ \X \psd 0 $ ($ \X \in \bbS^D_+$);
        2) All eigenvalues of $ \X$ are nonnegative
        ($\lambda_i(\X) \geq0$, $ i = 1,\cdots,D$);
        and 3) $\forany \bu \in \Real^D$, $ \bu^\T \X \bu \geq 0$.

    \subsection{A Theorem on Trace-one Semidefinite Matrices}

        Before we present our main results,
        we introduce an important theorem
        that serves the theoretical basis of \BoostMetric.

\begin{definition}
   \label{def:convexhull}
   For any positive integer $ m $,
   given a set of points
              $\{ \bx_1,...,\bx_m \} $
   in a real vector or matrix space $ \Sp $,
   the {\em convex hull} of $ \Sp $ spanned by $ m $
   elements in $ \Sp $
   is defined as:   $$
   \convhull_m ( \Sp ) =
   \left\{
   {\tsum_{i=1}^m w_i \bx_i}
   \Bigl| \,  w_i \geq 0,
   {\tsum_{i=1}^m w_i = 1},
   \bx_i \in \Sp
   \Bigr.
   \right\}.
   $$
   Define the linear  convex span of $ \Sp $
   as:\footnote{With slight abuse of notation,
   we also use the symbol $\convhull(\cdot)$ to denote
    convex span. In general it is not a convex hull.}
   \begin{align*}
   &\convhull( \Sp ) =
        \bigcup_m \convhull_m ( \Sp ) 
   =
   \left\{
   { \tsum_{i=1}^m w_i \bx_i}
   \Bigl| \,  w_i \geq 0,
   { \tsum_{i=1}^m w_i = 1},
   \bx_i \in \Sp,
   m \in \bbZ_+
   \Bigr.
   \right\}.
   \end{align*}
   Here $ \bbZ_+$ denotes the set of all positive integers.
   \end{definition}
   \begin{definition}
      Let us define $ \Gamma_1 $ to be the space of all positive semidefinite
      matrices $ \X \in \bbS^D_+ $ with trace equaling one:
      $$
      \Gamma_1 = \left\{ \X
      \left| \,
      \X \psd 0, \trace(\X) = 1
      \right.
      \right\};
      %
      %
      %
      %
      %
      %
      %
      %
      $$
      and
      $ \Psi_1 $ to be the space of all positive semidefinite
      matrices with both trace and rank equaling one:
      $$
      \Psi_1 = \left\{ \Z
      \left| \,
      \Z \psd 0, \trace(\Z) = 1, \rank(\Z) = 1
      \right.
      \right\}.
      $$
      We also define
      $ \Gamma_2 $ as the convex hull of $ \Psi_1$, \ie,
      $$
                  \Gamma_2 = \convhull(\Psi_1).
      $$
\end{definition}
\begin{lemma}
      Let $ \Psi_2 $ be a convex polytope defined as
      $  \Psi_2 = \{ \blambda  \in \Real^D | \, \lambda_k \geq 0
      $, $ \forall k = 0,\cdots,D$, $ \sum_{k=1}^D  \lambda_k = 1 \}$,
      then the points with only one element equaling one and all the others
      being zeros are the extreme points (vertexes) of $ \Psi_2 $.
      All the other points can not be extreme points.
\label{lem:main0}
\end{lemma}
\begin{proof}
      Without loss of generality, let us consider such a point $ \blambda'
      = \{1, 0, \cdots, 0\} $.
      If $ \blambda'$ is not an extreme point of $ \Psi_2 $, then it
      must be 
      possible to express it as a
      convex combination of a set of
      {\em other} points
      in $ \Psi_2 $: $ \blambda' = \sum_{i=1}^m w_i \blambda^i
      $, $ w_i > 0 $, $ \sum_{i=1}^m  w_i = 1$ and
      $ \blambda^i \neq \blambda'$.
      Then we have equations:
      $ \sum_{i=1}^m  w_i \lambda^i_k = 0$, $\forany k =
      2,\cdots, D $. It follows that $ \lambda^i_k = 0 $,
      $ \forany i $ and $ k = 2, \cdots, D$.
      That means, $ \lambda^i_1 = 1 $  $ \forany i $. This is
      inconsistent with $ \blambda^i \neq \blambda' $.
      Therefore  such a convex combination does not exist and
      $ \blambda' $ must be an extreme point.
      It is trivial to see that any $ \blambda $ that has more than
      one active element is an convex combination of the above-defined
      extreme points. So they can not be extreme points.
\end{proof}

\begin{theorem}
\label{thm:main1}
      $\Gamma_1$ equals to $\Gamma_2$; \ie,
       $ \Gamma_1 $ is also the convex hull of $ \Psi_1 $.
       In other words, all $ \Z \in \Psi_1 $,
       form the set of  extreme points of $ \Gamma_1 $.
\end{theorem}
\begin{proof}
   It is easy to check that any convex combination $ \sum_i{ w_i
   \Z_i }$,
   such that $ \Z_i \in \Psi_1$,
   resides in $ \Gamma_1 $, with the following two facts: 
   1)
   a convex combination of \PSD matrices is still a \PSD matrix; 
   2)
   $\trace\bigl(\sum_i{ w_i \Z_i }\bigr) =
    \tsum_i  w_i \trace( \Z_i)  =1 $.

   By denoting $ \lambda_1 \geq \cdots \geq \lambda_D \geq 0
   $ the eigenvalues of a $ \Z \in \Gamma_1$, we know that $
   \lambda_1 \leq 1$ because $ \sum_{i=1}^D \lambda_i = \trace(\Z) =
   1$. Therefore, all eigenvalues of $ \Z$ must satisfy: $ \lambda_i
   \in [0,1]$, $ \forall i=1,\cdots,D$ and $ \sum_i^D \lambda_i = 1$.
   By looking at the eigenvalues of $ \Z $ and using
   Lemma~\ref{lem:main0},
   it is immediate to see that a matrix $ \Z $ such that
   $ \Z \psd 0 $, $ \trace(\Z)=1$ and $\rank(\Z) > 1 $ can not
   be an extreme point of $ \Gamma_1 $.
   The only candidates for extreme points are those rank-one matrices
   ($\lambda_1 = 1 $ and $ \lambda_{2,\cdots,D} = 0$).
   Moreover, it is not possible that some rank-one matrices are
   extreme points and others are not because the other two constraints
   $ \Z \psd 0 $ and $ \trace( \Z ) = 1 $ do not distinguish
   between different rank-one matrices.

   Hence,  all $ \Z \in \Psi_1 $
       form the set of  extreme points of $ \Gamma_1 $.
   Furthermore, $ \Gamma_1 $ is a convex and compact set, which must
   have extreme points.
   The Krein-Milman Theorem \cite{Krein1940Extreme} tells us that a convex
   and compact set is equal to the convex hull of its extreme points.
\end{proof}
   This theorem is a special case of the results from
   \cite{Overton1992Sum} in the context of eigenvalue optimization.
   A different proof for the above theorem's general version  can also
   be found in \cite{Fillmore1971}.

   In the context of semidefinite optimization, what is of interest about
   Theorem~\ref{thm:main1} is as follows: it tells us that a bounded
   \PSD matrix constraint $\X \in \Gamma_1 $ can be equivalently
   replaced with a set of constrains which belong to $ \Gamma_2 $.  At
   the first glance, this is a highly counterintuitive proposition
   because $ \Gamma_2 $
   involves many more complicated constraints. Both $w_i $ and
   $ \Z_i$ ($\forany i = 1,\cdots,m$) are unknown variables. Even
   worse, $ m $ could be extremely (or even infinitely) large.
        Nevertheless, this is the type of problems that
        {\em boosting} algorithms are designed to solve.
        Let us give a brief overview of boosting algorithms.

\subsection{Boosting}

      Boosting is an example of ensemble learning, where multiple
      learners are trained to solve the same problem.
      Typically a boosting algorithm \cite{Schapire1999Boosting}
      creates a single strong learner
      by incrementally adding base (weak) learners to the final strong
      learner. The base learner has an important impact on the strong
      learner.
      In general, a boosting algorithm builds on a user-specified
      base learning procedure and runs it repeatedly on modified data
      that are outputs from the previous iterations.

      The general form of the boosting algorithm 
      is sketched in
       Algorithm~\ref{alg:BOOST}.
      The inputs to a boosting algorithm are a set of training example
      $ \bx $, and their corresponding class labels $y$.
      The final output is a strong classifier which takes the form
      \begin{equation}
      F_{\btheta} (\bx)  = \tsum_{ j =1}^J w_j h_j(\bx).
      \label{eq:boost0}
      \end{equation}
      Here $ h_j( \cdot )$ is a base learner.
      From Theorem~\ref{thm:main1}, we know that a matrix
      $ \X \in \Gamma_1$ can be decomposed as
      \begin{equation}
        \X = \tsum_{ j=1 }^J w_j
        \Z_j,  \Z_j \in \Gamma_2.
        \label{eq:sdp0}
      \end{equation}
      By observing the similarity between Equations~\eqref{eq:boost0}
      and \eqref{eq:sdp0}, we may view $ \Z_j $ as a weak classifier
      and the matrix $ \X $ as the strong classifier that  we want to learn.
      This is exactly the problem that boosting methods have been
      designed to solve.
      This observation inspires us to solve a special type of
      semidefinite optimization problem
      using boosting techniques.

      The sparse greedy approximation algorithm proposed by
      \citet{Zhang2003Sequential} is an efficient method for solving a
      class of convex problems, and achieves fast convergence rates.
      It has also been shown  that boosting algorithms can be
      interpreted within the general framework of
      \cite{Zhang2003Sequential}.  The main idea of sequential greedy
      approximation, therefore, is as follows.  Given an
      initialization $ \bu_0 $, which is in a convex
      subset of a linear vector space,
      a matrix space or a functional space,
      the algorithm finds $\bu_i $ and $ \lambda \in (0, 1) $
      such that the objective function
      $ F ( (1 -  \lambda ) \bu_{i - 1} + \lambda \bu_i ) $
      is minimized.
      Then the solution $ \bu_i $ is updated as
      $ \bu_i = (1 -  \lambda ) \bu_{i-1} + \lambda \bu_i $ and the
      iteration goes on. Clearly, $ \bu_i $ must remain in the original space.
      As shown next, our first case, which learns a metric using
      the hinge loss, greatly resembles this idea.

      \SetVline
      \linesnumbered
      \begin{algorithm}[t]
      \caption{The general framework of boosting.}
      \begin{algorithmic}
      \KwIn{Training data.}
      {
      Initialize a weight set $ \bu  $ on the training examples\;
      }
     \For{$ j = 1,2,\cdots,$}
       {
         $ \Dot $  Receive a weak hypothesis $ h_j(\cdot)$\;
         $ \Dot $  Calculate $ w_j > 0$\;
         $ \Dot $  Update $\bu$.
     }
     \KwOut{
         A convex combination of the weak h\-y\-p\-o\-th\-e\-s\-es:
         $ F_{\bw} (\bx)  = \tsum_{j=1}^J w_j h_j(\bx)$.
      }
   \end{algorithmic}
   \label{alg:BOOST}
   \end{algorithm}

\subsection{Distance Metric Learning Using Proximity Comparison}

        The process of measuring distance using a
        Mahalanobis metric is equivalent to
        linearly transforming the data by a projection matrix $\L \in
        \mathbb R^{D \times d}$ (usually $ D \geq d $)
        before calculating the standard Euclidean distance:
			  \begin{align}
            \label{EQ:1}
            \dist_{ij}^2 &=\|\L^\T \ba_i - \L^{\T} \ba_j\|^2_2  
             = (\ba_i - \ba_j)^{\T} \L \L^{\T} (\ba_i - \ba_j)
             = (\ba_i - \ba_j)^{\T} \X (\ba_i - \ba_j).
        \end{align}
        As described above, the problem of learning a Mahalanobis metric 
        can be approached in terms of learning the matrix $\L$, or the \PSD
        matrix $\X$.
        If $ \X = \I $, 
        the Mahalanobis distance
        reduces to the Euclidean distance. If $ \X$
        is diagonal, the problem corresponds to learning a metric in
        which  different features are given different weights, \aka
        feature weighting.
		Our approach is to learn a full \PSD matrix $\X$, 
        however, using \BoostMetric.

         In the framework of large-margin learning,
         we want to maximize the distance between $\dist_{ij}$ and
         $\dist_{ik}$. That is, we wish to make $\dist_{ik}^2 -
         \dist_{ij}^2 = (\ba_i - \ba_k)^{\T} \X (
         \ba_i - \ba_k) - (\ba_i - \ba_j)^{\T}
         \X ( \ba_i - \ba_j)$ as large as possible under some
         regularization.
         To simplify notation, we rewrite the distance between
         $\dist_{ij}^2$ and $\dist_{ik}^2$ as
         $
          {\dist}_{ik}^2 -
          {\dist}_{ij}^2 = \innerp{\A_r}{\X},
         $
         where
         \begin{equation}
         \A_r =
         (\ba_i - \ba_k) (
         \ba_i - \ba_k)^{\T} - (\ba_i - \ba_j)
         ( \ba_i - \ba_j)^{\T},
             \label{EQ:Ar}
         \end{equation}
         for
         $r = 1,\cdots,|\SS|$
         and
         $ |\SS| $ is the size of the set of constraints $ \SS$
         defined in Equation~\eqref{eq:SSDefn}.

\section{Algorithms}
\label{sec:algorithm}

    In this section, we define the optimization problems for metric learning.
    We mainly investigate the cases using the hinge loss, exponential
    loss and logistic loss
    functions.
    In order to derive an efficient optimization
    strategy, we look at their Lagrange dual problems and design
    boosting-like approaches for efficiency.

\subsection{Learning with the Hinge Loss}
\label{sec:hinge}

      Our goal is to derive a general algorithm for \PSD matrix
      learning with the hinge loss function.  Assume that we want to
      find a \PSD matrix $ \X \psd 0$ such that a set of constraints
      \[ \innerp{\A_r}{\X} > 0, r = 1,2,\cdots, \] are satisfied as
      {\em well} as possible.  Here $ \A_r $ is as defined in
      \eqref{EQ:Ar}.  These constraints need not all be strictly
      satisfied and thus we define the margin $\rho_r =
      \innerp{\A_r}{\X}$, $ \forany r$.

      Putting it into the maximum margin learning framework, we want
      to minimize the following trace norm regularized objective
      function:
      $
           \sum_r F ( \innerp{\A_r}{\X} )  + v \trace(\X),
      $
     with  $ F ( \cdot ) $ a convex loss function and $v$ a
     regularization constant.   Here we have used the trace norm
     regularization. 
     Of course a Frobenius norm regularization term can also be used
     here.  Minimizing the Frobenius norm $ || \X ||_{\rm F}^2  $,
     which is equivalent to minimize the $ \ell_2 $ norm of the
     eigenvalues of $ \X $, penalizes a solution that is far away from
     the identity matrix.
     With the hinge loss, we can write the optimization problem as: 
    \begin{align}
        \label{EQ:hinge_a}
        \max_{ \rho, \X,  {\boldsymbol \xi} } 
        \;   \rho -   v   \tsum_{r=1}^{ |\cal I |  }   
                      \xi_r, \,\, \sst 
                      \,
                      \innerp{\A_r}{\X} \geq \rho - \xi_r, \forany r;
                      \X \psd 0, \trace(\X) = 1; 
                      \; {\boldsymbol \xi}
                      \geq {\boldsymbol 0}. 
    \end{align}
     Here $ \trace(\X) = 1 $ removes the scale ambiguity because the
     distance inequalities are scale invariant.

      We can decompose $ \X$ into:
      $
           \X = \tsum_{j =1}^J  w_j \Z_j,
      $
      with
      $w_j >  0$, $\rank(\Z_j) = 1$ and
      $\trace(\Z_j)= 1$, $ \forany j$.
      So we have
      \begin{align}
             \innerp{\A_r}{\X}
                  = \innerp{\A_r}{ \tsum_{j=1}^J w_j\Z_j }
                  = \tsum_{j=1}^J w_j \innerp{\A_r}{\Z_j}
                  = \tsum_{j=1}^J w_j \H_{rj} = \H_{ r: } \bw, \forany r.
      \label{EQ:symb}
      \end{align}
        Here $\H_{rj}$ is a shorthand for $\H_{rj} =
        \innerp{\A_r}{\Z_j}$.
        Clearly, $ \trace(\X) = {\boldsymbol 1}^\T \bw $.
        Using Theorem   
        \ref{thm:main1}, we  replace the p.s.d.\   conic constraint 
        in the primal \eqref{EQ:hinge_a} with a linear convex
        combination of rank-one unitary matrices: 
      $
           \X = \tsum_j 
               w_j \Z_j,
      $
      and $ {\boldsymbol 1}^\T \bw = 1$. 
      Substituting $ \X $ in \eqref{EQ:hinge_a}, we have 
      \begin{align}
          \max_{ \rho, \bw, {\boldsymbol \xi}}  
          \, \rho - v \tsum_{r=1} ^ { |\SS | } \xi_r, 
          \, \sst 
          \, \H_{r:} \bw \geq \rho - \xi_r, 
          (r = 1, \dots, | {\cal I }|);
          \bw \geq {\boldsymbol 0}, {\boldsymbol
          1}^\T \bw = 1; \; {\boldsymbol \xi} \geq {\boldsymbol 0}.
          \label{EQ:HingeP2}
      \end{align}

      The Lagrange dual problem of   
      the above linear programming problem \eqref{EQ:HingeP2} is easily derived:
        \begin{align}
            \label{EQ:HingeD2}
            \min_{ \pi, \bu } \; \pi
            \; \sst \,
            \tsum_{r = 1}^ { | \SS | } u_r  \H_{r:}  \leq \pi
            {\boldsymbol 1} ^\T; 
            {\boldsymbol 1}^\T \bu = 1,  
            {\boldsymbol 0} \leq \bu \leq v {\boldsymbol 1}. 
        \end{align}
        We can then use column generation to solve the original
        problem iteratively by looking at both the primal and dual
        problems. 
        See \citet{Shen2008PSD} for the algorithmic details. 
        In this work we are more interested in smooth loss functions 
        such as the exponential loss and logistic loss, as presented in
        the sequel.

\subsection{Learning with the Exponential Loss}
\label{sec:Exponential}

      By employing the exponential loss, we want to optimize
      \begin{align}
      \label{EQ:4}
      \min_{   \X, {\boldsymbol \rho}  } \, & \log
                    \bigl(
                           \tsum_{r=1}^{|\SS|}  \exp (- \rho_r )
                    \bigr) +  v
        \trace(\X)
        \notag
        \\
         \sst \, & \rho_r = \innerp{\A_r}{\X}, r =
         1,\cdots,|\SS|,
        \; \X \psd 0.
      \end{align}
   Note that:
         1) We 
         are proposing a
         logarithmic version of the sum of
         exponential loss. This transform does not change the original
         optimization problem of sum of exponential loss because the
         logarithmic function is strictly monotonically increasing.
      2)
         A regularization term $\trace(\X)$ has been applied. Without
         this regularization, one can always multiply 
         $\X$ by an arbitrarily large scale factor in order
         to make the exponential loss approach
         zero in the case of all constraints being satisfied.
         This trace-norm regularization may also lead to low-rank
         solutions.
     3)
         An auxiliary variable $ \rho_r, r = 1,\dots $ must be
         introduced for deriving a meaningful dual problem, as we show
         later.

      We now derive the Lagrange dual of the problem that we are interested in.
      The original problem \eqref{EQ:4} now becomes
      \begin{align}
            \label{EQ:5}
            \min_{ {\boldsymbol \rho}, \bw  } \, & \log
                    \bigl(
                           \tsum_{r=1}^{|\SS|}  \exp (- \rho_r )
                    \bigr) +  v
                    {\boldsymbol 1}^\T \bw
        \notag
        \\
        \sst \, & \rho_r = \H_{r:} \bw, r =
         1,\cdots,|\SS|;
        \, 
        \bw \geq \boldsymbol 0.
      \end{align}
      We have used the Equation \eqref{EQ:symb}. 
        In order to derive its dual, we write its Lagrangian
        \begin{align}
           L( \bw, \boldsymbol \rho, \bu  )
           &=
            \log
                    \bigl(
                           \tsum_{r=1}^{|\SS|}  \exp (- \rho_r)
                    \bigr) +  v
        {\boldsymbol 1}^\T \bw
        +
        \tsum_{r=1}^{ |\SS| } u_r ( \rho_r - \H_{r:} \bw ) -
        \bp^\T \bw,
        \end{align}
        with $ \bp \geq  0 $. The dual problem is obtained by finding
        the saddle point of $ L $; \ie,
        $ \sup_\bu \inf_{ \bw, \boldsymbol
        \rho} L $.
        \begin{align}
        \inf_{\bw, \boldsymbol \rho}
        L  &=
        \inf_{\boldsymbol \rho}
        \overbrace{
         \log
                    \bigl(
                           \tsum_{r=1}^{|\SS|}  \exp ( - \rho_r )
                    \bigr)
                     + \bu^\T {\boldsymbol \rho}
                     }^{L_1}
        +
        \inf_{\bw}
        \overbrace{
                (
                v{\boldsymbol 1}^\T - \tsum_{r=1}^{ |\SS| } u_r \H_{r:}
                -\bp^\T
                ) \bw
                }^{ L_2 }
                \label{EQ:LL2}
        \\
        &=
        - \tsum_{ r = 1}^{ |\SS|} u_r \log u_r.
        \end{align}
        The infimum of $ L_1 $ is found by setting its first derivative to
        zero and we have:
        \begin{equation}
        \inf_{\boldsymbol \rho} L_1
        =
        \begin{cases}
            - \tsum_r u_r \log u_r & \text{if $ \bu \geq  {\boldsymbol
            0}, {\boldsymbol 1}^\T \bu =1
         $,}
        \\
        - \infty                & \text{otherwise.}
        \end{cases}
        \end{equation}
        The infimum is Shannon entropy.
        $ L_2 $ is linear in $ \bw $, hence it must be $ \boldsymbol 0
        $. It leads to
        \begin{equation}
             \tsum_{r=1}^{ |\SS| } u_r \H_{r:}
             \leq v{\boldsymbol 1}^\T.
            \label{EQ:C1}
        \end{equation}
        The Lagrange dual problem of \eqref{EQ:5}
        is an entropy maximization problem, which writes
        \begin{align}
            \max_\bu \, &  - \tsum_{r=1}^{ |\SS| } u_r \log u_r, \,
            \sst \, \bu \geq {\boldsymbol 0},
            {\boldsymbol 1}^\T \bu =1,
            \text{and } \eqref{EQ:C1}.
            \label{EQ:D1}
        \end{align}
        Weak and strong duality hold under mild conditions
        \cite{Boyd2004Convex}. That
        means,
        one can usually solve one problem from the other.
        The KKT conditions link the optimal between these two
        problems. In our case, it is
        \begin{equation}
            u_r^\star =
                      \frac{ \exp ( - \rho_r^\star ) }
                      { \tsum_{k=1}^{ |\SS| } \exp ( - \rho_k^\star ) },
                      \forany r.
            \label{EQ:KKT}
        \end{equation}

        While it is possible to devise a totally-corrective column
        generation based optimization procedure for solving our
        problem as the case of LPBoost \cite{Demiriz2002LPBoost},
        we are more interested in considering {\em
        one-at-a-time} coordinate-wise descent algorithms,
        as the case of AdaBoost \cite{Schapire1999Boosting}.
        Let us start from some basic knowledge of column generation
        because our coordinate descent strategy is inspired by column
        generation.

        If we know all the bases $ \Z_j \, (j=1\dots J) $ and hence the
        entire matrix $ \H $ is known. Then either the primal
        \eqref{EQ:5} or the dual \eqref{EQ:D1}
        can be trivially solved (at least in theory) because both are
        convex optimization problems. We can solve them in polynomial
        time. Especially the primal problem is  convex minimization
        with simple non\-negative\-ness constraints. Off-the-shelf
        software like LBFGS-B \cite{Zhu1997LBFGS} can be used for this
        purpose.
        Unfortunately, in practice, we do not access
        all the bases: the possibility of $ \Z $ is
        infinite. In convex optimization, column generation is a
        technique that is designed for solving this difficulty.

        Column generation was originally advocated  for solving large
        scale linear programs \cite{Lubbecke2005Selected}.
        Column generation is based on the fact that for a linear
        program, the number of non-zero variables of the optimal
        solution is equal to the number of constraints. Therefore,
        although
        the number of possible variables may be large, we only need a
        small subset of these in the optimal solution.
        For a general convex problem, we can use column generation
        to obtain
        an {\em approximate} solution.
        It works by only considering a small
        subset of the entire variable set. Once it is solved, we ask
        the question:``Are there any other variables that can be
        included to improve the solution?''.  So we must be able to
        solve the subproblem: given a set of dual values, one either
        identifies a variable that has a favorable reduced cost, or
        indicates that such a variable does not exist.
        Essentially, column generation
        finds the variables with negative reduced costs without
        explicitly enumerating all variables.

        Instead of directly solving the primal problem \eqref{EQ:5},
        we find the most
        violated constraint in the dual \eqref{EQ:D1}
        iteratively for the current
        solution and adds this constraint to the optimization problem.
        For this purpose, we need to solve
        \begin{equation}
        \label{EQ:weak}
          {\hat  \Z}
          = \argmax\nolimits_\Z \left\{ \tsum_{r=1}^{ |\SS| }
                       u_r
                       \bigl< \A_r,  \Z
                       \bigr>,
                       \, \sst \,
                       \Z  \in  \Psi_1
                      \right\}.
         \end{equation}
        We discuss how to efficiently solve \eqref{EQ:weak} later.
        Now we move on to derive a coordinate descent optimization
        procedure.

        \subsection{Coordinate Descent Optimization}
        We show how an AdaBoost-like optimization
        procedure can be derived.

        \subsubsection{Optimizing for $ w_j $ }

        Since we are interested in the {\em one-at-a-time}
        coordinate-wise optimization, we keep $ w_1, $ $w_2,$ $\dots,$
        $ w_{j-1} $ fixed when solving for $ w_j $.
        The cost function of the primal problem is (in the following
        derivation, we drop
        those terms irrelevant to the variable $ w_j $)
        \[
                C_p ( w_j ) =  \log \bigl[ \tsum_{r=1}^{|\SS|}
                \exp ( -\rho_r^{j-1} ) \cdot
                \exp ( - \H_{rj} w_{j}  )
                \bigr] + v w_j.
        \]
        Clearly, $ C_p $ is convex in $ w_j $ and hence there is only
        one  minimum that is also globally optimal.
        The first derivative of $C_p$ \wrt $ w_j $ vanishes at
        optimality,
        which results in
        \begin{equation}
            \tsum_{r=1}^{ |\SS| }  ( \H_{rj} - v )u_r^{j-1}
                                  \exp( -w_j \H_{rj}  ) = 0.
            \label{EQ:w}
        \end{equation}

        If $ \H_{rj} $  is  discrete, such as $\{+1, -1 \}$ in standard
        AdaBoost, we can obtain a closed-form solution similar to
        AdaBoost. Unfortunately in our case, $ \H_{rj} $ can be any
        real value.
        We instead use bisection to search for the optimal $ w_j $.
        The bisection method is one of the root-finding algorithms.
        It repeatedly divides an interval in half and then selects the
        subinterval in which a root exists.  Bisection is a simple and
        robust, although it is not the fastest algorithm for
        root-finding.
        Algorithm \ref{ALG:bisection}
        gives the bisection procedure. We have utilized the fact that
        the l.h.s. of \eqref{EQ:w} must be positive at $ w_l $.
        Otherwise no solution can be found.  When $ w_j = 0 $,
        clearly the l.h.s. of \eqref{EQ:w} is positive.

\SetVline
\linesnumbered

\begin{algorithm}[t]
\caption{Bisection search for $ w_j $.}
\begin{algorithmic}
\normalsize{
   \KwIn{
    An interval $ [ w_l, w_u] $ known to contain the optimal value of
    $ w_j $ and convergence tolerance $ \varepsilon > 0$.
   }
\Repeat{$w_u - w_l < \varepsilon$}
{
$\Dot$ 
        $w_j = 0.5 ( w_l + w_u )$\;
$\Dot$
\If { $\text{{\rm l.h.s.} of} \,\, \eqref{EQ:w}  > 0 $}
        { $ w_l = w_j $;  }
\Else
        { $ w_u = w_j $.  }
}
\KwOut{
        $w_j$.
}
}
\end{algorithmic}
\label{ALG:bisection}
\end{algorithm}

        \subsubsection{Updating $ \bu $ }
        The rule for updating $ \bu $ can be easily obtained from
        \eqref{EQ:KKT}.
        At iteration $ j $, we have
        \begin{align*}
            u_r^j \propto \exp ( - \rho_r^j  )
            \propto u_r^{j-1} \exp  (- \H_{rj} w_j ) ,
            \text{ and }
             \tsum_{r=1}^{ |\SS| } u_r^j = 1,
        \end{align*}
        derived from \eqref{EQ:KKT}.
        So once $ w_j $ is calculated,
        we can update $ \bu $ as
        \begin{equation}
            \label{EQ:UpdateRule}
            u_r^j = \frac{ u_r^{j-1} \exp  (- \H_{rj} w_j )  }{ z },
            r = 1,\dots, |\SS|,
        \end{equation}
        where $ z $ is a normalization factor so that $ \tsum_{r=1}^{
        |\SS| } u_r^j = 1 $.
        This is exactly the same as AdaBoost.

\subsection{The Base Learning Algorithm}

   In this section, we show that
   the optimization problem \eqref{EQ:weak} can be exactly and
   efficiently solved using eigenvalue-decomposition (EVD). 
   
   From $ \Z \psd 0 $ and $ \rank(\Z) = 1$, we know that $ \Z $ has
   the format: $ \Z = \bv \bv^\T$,  $ \bv \in \Real^D$; and $ \trace(\Z)
   = 1$ means $ \fnorm{2}{\bv} = 1$. 
   We have $$   
            \bigl<     
                          {\tsum_{r=1}^{| \SS |}
                            {  u_r \A_r  }
                           },
                           \Z
                       \bigr>
                =  \bv \bigl(  {\tsum_{r=1}^{| \SS |}
                            {  u_r \A_r  } }   
                        \bigr) \bv^\T.
         $$ 
   By denoting  
   \begin{equation}
   \label{EQ:MM}
   {\hat \A} =  \tsum_{r=1}^{| \SS |}
                            {  u_r \A_r  },
   \end{equation}
   the base learning optimization equals:
   \begin{align}
       \max_\bv \,\, \bv^\T {\hat \A} \bv, \,\, \sst \fnorm{2}{\bv} = 1.
       \label{EQ:MM2}
   \end{align}
   It is clear that the largest eigenvalue of $ \hat \A $, 
   $ \eigenmax (\hat \A) $, and its corresponding eigenvector $ \bv_1 $
   gives the solution to the above problem.
   Note that $ \hat \A $ is symmetric.
        
   $ \eigenmax (\hat \A) $
   is also used as one of the stopping criteria of the algorithm.
   Form the condition \eqref{EQ:C1},  
   $ \eigenmax (\hat \A) < v $ means that
        we are not able to find a new base matrix 
   $\hat \Z $ that violates
   \eqref{EQ:C1}---the algorithm converges.

        Eigenvalue decompositions is one of the main computational
        costs in
        our algorithm.  There are approximate eigenvalue solvers,
        which guarantee that for a symmetric matrix $ \U $ and any $
        \varepsilon > 0 $, a vector $ \bv $ is found such that  $ \bv
        ^\T\U \bv \geq \eigenmax - \varepsilon$.  To approximately
        find the largest eigenvalue and eigenvector can be very
        efficient using Lanczos or power method.  We can use the
        MATLAB function {\sf eigs} to calculate the largest
        eigenvector, which calls mex files of ARPACK.  ARPACK is a
        collection of Fortran subroutines designed to solve large
        scale eigenvalue problems.  When the input matrix is
        symmetric,  this software uses a variant of the Lanczos
        process called the implicitly restarted Lanczos method.

        Another way to reduce the time for computing the leading
        eigenvector is to compute an approximate EVD by a fast Monte
        Carlo algorithm such as the linear time SVD
        algorithm developed in \cite{Drineas2004Fast}.

\SetVline
\linesnumbered

\begin{algorithm}[t]
\caption{Positive semidefinite matrix learning with stage-wise boosting.}
\begin{minipage}[t!]{0.98\linewidth}
\begin{algorithmic}
\normalsize{
   \KwIn{
    \begin{itemize}
       \item
           Training set triplets $  ( \ba_i, \ba_j, \ba_k ) \in \SS $;
           Compute $ \A_r, r = 1,2,\cdots,$ using \eqref{EQ:Ar}. 
        \item
            $ J $: maximum number of iterations;
        \item
            (optional) regularization parameter $ v $; We may
            si\-m\-p\-ly set
            $ v $ to a v\-e\-r\-y small value, \eg, $ 10^{-7}$.  
    \end{itemize}
   }
   { {\bf Initialize}:
        $ u^0_r = \tfrac{1}{|\SS|}, r = 1\cdots |\SS|  $\;
    }%
\For{ $ j = 1,2,\cdots, J $ }
{
$\Dot$ 
        Find a new base $ \Z_j $ by finding the largest
        eigenvalue ($\eigenmax (\hat \A)$) and its eigenvector of $ \hat \A $ in 
        \eqref{EQ:MM}\;

$\Dot$
        \If{
        $\eigenmax (\hat \A) < v $}
        {break (converged)\;}
$\Dot$
Compute $ w_j $ using Algorithm \ref{ALG:bisection}\;
$\Dot$
Update $ \bu$ to obtain $ u_r^j, r = 1,\cdots |\SS| $ using
\eqref{EQ:UpdateRule}\;
}
\KwOut{
The final \PSD matrix $ \X \in \Real^{D \times D}$, $ \X =
\sum_{j=1}^J w_j  \Z_j $.  
}
}
\end{algorithmic}
\end{minipage}
\label{ALG:MetricBoost}
\end{algorithm}

        We summarize our main algorithmic results in Algorithm~\ref{ALG:MetricBoost}.

\subsection{Learning with the Logistic Loss}

    We have considered the exponential loss in the last content. The
    proposed framework is so  general that
    it can also accommodate other  convex loss  functions. Here we consider the logistic loss, which
    penalizes mis-classi\-fications with more moderate penalties than the exponential loss. It is
    believed on noisy data, the logistic loss may achieve better classification performance.

    With  the same settings as in the case of the exponential loss,
    we can write our optimization problem as
   \begin{align}
      \label{EQ:4B}
      \min_{ {\boldsymbol \rho}, \bw  }\, &
                \tsum_{r=1}^{|\SS|}  \logit ( \rho_r )
                     +  v
                     {\bf 1}^\T \bw
        \notag
        \\
        \sst \, & \rho_r = \H_{r:} \bw, r =
         1,\cdots,|\SS|,
        \bw \geq 0.
      \end{align}
      Here $ \logit ( \cdot )$ is the logistic loss defined as
      $ \logit( z ) = \log(1 + \exp(-z) )$.
        Similarly, we derive its Lagrange dual as
 \begin{align}
      \label{EQ:4C}
      \min_{ \bu } \, &
                \tsum_{r=1}^{|\SS|}  \logit^* ( - u_r )
        \notag
        \\
         \sst \, &
         \tsum_{r=1}^{|\SS|} u_r \H_{r:} \leq  v {\bf 1}^\T,
      \end{align}
  where $ \logit^* (\cdot)$ is the Fenchel conjugate function of
  $ \logit ( \cdot ) $, defined as
  \begin{equation}
      \label{EQ:fenchel}
      \logit^*(- u) = u \log(u) + ( 1 - u ) \log(1 - u),
  \end{equation}
  when $ 0 \leq u \leq 1$, and $ \infty $ otherwise.
  So the Fenchel conjugate of $ \logit(\cdot)$ is the binary
  entropy function. We have reversed the sign of $ \bu $ when deriving the
  dual.

    Again, according to the KKT conditions,
    we have
    \begin{equation}
        \label{EQ:KKT3}
        u_r^\star = \frac{ \exp ( - \rho_r^\star )  }{ 1 + \exp  (- \rho_r^\star ) },
        \;\; \forany r,
    \end{equation}
    at optimality.
    From \eqref{EQ:KKT3} we can also see that $ u $ must be in $ (0, 1)$.

    Similarly, we want to optimize the primal cost function in a coordinate descent way.
    First, let us find the relationship between $ u_r ^j $ and
    $ u_r ^{j-1} $. Here $ j $ is the iteration index.
    From \eqref{EQ:KKT3}, it is trivial to obtain
    \begin{equation}
        \label{EQ:UpdateU1}
        u^j_r = \frac{1}{ (1/u_r^{j-1} -1 ) \exp ( \H_{rj} w_j ) + 1},
        \;\; \forany r.
    \end{equation}
    The optimization of $ w_j $ can be solved by looking for the root of
    \begin{equation}
        \label{EQ:w2}
        \tsum_{r=1}^{|\SS|}  \H_{rj} u_r^j - v = 0,
    \end{equation}
    where $ u_r^j $ is a function of $ w_j $ as defined in \eqref{EQ:UpdateU1}.

    Therefore, in the case of the logistic loss, to find $ w_j $, we modify the bisection
    search of Algorithm \ref{ALG:bisection}:
    \begin{itemize}
        \item Line 3: {\bf if} l.h.s. $of \; \eqref{EQ:w2} > 0$ {\bf then} \dots
    \end{itemize}
    and Line 7 of Algorithm \ref{ALG:MetricBoost}:
    \begin{itemize}
        \item Line 7: Update $ \bu$ using \eqref{EQ:UpdateU1}.
    \end{itemize}

\SetVline
\linesnumbered

\begin{algorithm}[t]
\caption{Positive semidefinite matrix learning with totally corrective boosting.}
\begin{minipage}[t!]{0.98\linewidth}
\begin{algorithmic}
\normalsize{
   \KwIn{
    \begin{itemize}
       \item
           Training set triplets $  ( \ba_i, \ba_j, \ba_k ) \in \SS $;
           Compute $ \A_r, r = 1,2,\cdots,$ using \eqref{EQ:Ar}. 
        \item
            $ J $: maximum number of iterations;
        \item
            Regularization parameter $ v $.  
    \end{itemize}
   }
    { {\bf Initialize}:
        $ u^0_r = \tfrac{1}{|\SS|}, r = 1\cdots |\SS|  $\;
    }%
\For{ $ j = 1,2,\cdots, J $ }
{
$\Dot$ 
        Find a new base $ \Z_j $ by finding the largest
        eigenvalue ($\eigenmax (\hat \A)$) and its eigenvector of $ \hat \A $ in 
        \eqref{EQ:MM}\;

$\Dot$
        \If{
        $\eigenmax (\hat \A) < v $}
        {break (converged)\;}
$\Dot$
    Optimize for $ w_1, w_2, \cdots, w_j $ by solving
    the primal problem \eqref{EQ:5} when the exponential loss is used
    or \eqref{EQ:4B} when the logistic loss is used\;
$\Dot$
    Update $ \bu$ to obtain $ u_r^j, r = 1,\cdots |\SS| $ using
    \eqref{EQ:KKT} (exponential loss) or \eqref{EQ:KKT3} (logistic loss)\;
}
\KwOut{
The final \PSD matrix $ \X \in \Real^{D \times D}$, $ \X =
\sum_{j=1}^J w_j  \Z_j $.  
}
}
\end{algorithmic}
\end{minipage}
\label{ALG:TC}
\end{algorithm}

\subsection{Totally Corrective Optimization}

        In this section, we derive a totally-corrective version of
        \BoostMetric, similar to the case of TotalBoost 
        \cite{Warmuth2006Total,Dual2010Shen} 
        for classification, in the sense that the coefficients of 
        all weak learners are updated at each iteration.

        Unlike the stage-wise optimization, here we do not need to keep previous
        weights of weak learners $ w_1, w_2, \dots, w_{j-1} $. Instead, the
        weights of all the selected weak learners $ w_1, w_2, \dots, w_{j} $
        are updated at each iteration $j$.
        As discussed, our learning procedure is able to employ various loss
        functions such as the hinge loss, exponential loss or logistic loss. To
        devise a totally-corrective optimization procedure for solving our
        problem efficiently, we need to ensure the object function to be
        differentiable with respect to the variables $w_1, w_2, \dots, w_j$.
        Here, we use the exponential loss function and the logistic loss
        function. It is possible to use sub-gradient descent methods when
        a non-smooth loss function like the hinge loss is used.

        It is clear that solving for $\boldsymbol w$  is a typical
        convex optimization problem since it has a differentiable and convex function
        \eqref{EQ:5} when the exponential loss is used, or \eqref{EQ:4B} when
        the logistic loss is used. Hence it can be solved using
        off-the-shelf gradient-descent solvers like L-BFGS-B \cite{Zhu1997LBFGS}.

        Since all the weights $w_1, w_2, \dots, w_j$ are updated,
        $u_r^j$ on $r=1 \dots |\SS|$ need  not to be updated but
        re-calculated at each iteration $j$. To calculate $u_r^j$, we use
        \eqref{EQ:KKT} (exponential loss) or \eqref{EQ:KKT3} (logistic loss)
        instead of \eqref{EQ:UpdateRule} or \eqref{EQ:UpdateU1}
        respectively.
        Totally-corrective \BoostMetric methods are very simple to implement.
        Algorithm \ref{ALG:TC} gives the summary of this algorithm.
Next, we show the convergence property of Algorithm \ref{ALG:TC}.  
Formally, we want to show the following theorem.
\begin{theorem}
                    Algorithm \ref{ALG:TC}  makes progress at each
                    iteration. In other words, the objective value
                    is decreased at each iteration. 
                    Therefore, in the limit, Algorithm \ref{ALG:TC}
                    solves the optimization problem \eqref{EQ:5} 
                    (or  \eqref{EQ:4B}) 
                    globally to a
                    desired accuracy.  
\end{theorem}
\begin{proof}
    Let us consider the exponential loss case of problem \eqref{EQ:5}. 
    The proof follows the same discussion for the logistic loss, or
    any other smooth convex loss function. Assume that the current
    solution is a finite subset of base learners (rank-one trace-one
    matrices) and their corresponding linear coefficients $ \bw $.
    If we add  a base matrix $ \hat \Z $ that is not in the current
    subset,  and the corresponding $ \hat w = 0 $,
    then the objective value and the solution must remain unchanged. 
    We can conclude that the current learned base learners and $
    \bw $ are the optimal solution already.

    Consider the case that this optimality condition is violated. We
    need to show that  we can find a base learner $ \hat \Z $, which is
    not in the current set of all the selected base learners,
    such that $ {\hat w} > 0 $ holds. 
    Now assume that $ \hat \Z $ is the base learner found by solving 
    \eqref{EQ:MM2}, and the convergence condition $  \lambda_{ \rm max
    } ( {\hat \A} )  \leq v  $ is not satisfied. So,  
    we have $  \lambda_{ \rm max } ( {\hat \A} ) = \innerp{
    \tsum_{r=1}^{ |\SS| } u_r \A_r} { \hat \Z }  > v  $. 

        If, after this weak learner $ \hat \Z $ is added into the
        primal problem, the primal solution remains unchanged, 
        i.e., the corresponding $ {\hat w} = 0$,
        then 
        from the optimality condition that $ L_2  $ in \eqref{EQ:LL2} must
        be zero, we know that  
        $  {\hat p}  =  v -    \innerp{
    \tsum_{r=1}^{ |\SS| } u_r \A_r} { \hat \Z } < 0 $.
    This contradicts the fact the Lagrange multiplier $ {\hat p} \geq
    0 $.

    We can conclude that after the base learner $ \hat \Z $ is added
    into the primal problem, its corresponding $ \hat w $ must admit a
    positive value. It means that one more free variable is added into
    the problem and re-solving the primal problem would reduce the
    objective value. Hence a strict decrease in the objective is
    guaranteed. So Algorithm \ref{ALG:TC} makes progress at each
    iteration.

    Furthermore, as the optimization problems involved are all convex,
    there are no local optimal solutions. Therefore Algorithm
    \ref{ALG:TC} is guaranteed to converge to the global solution. 

      Note that the above proof establishes the convergence of Algorithm
    \ref{ALG:TC} but it remains unclear about the convergence rate. 
\end{proof}

        \subsection{Multi-pass \BoostMetric}
 
       In this section, we show that \BoostMetric can use
       multi-pass learning to enhance the performance.

        Our \BoostMetric uses training set triplets $  ( \ba_i, \ba_j, \ba_k
        ) \in \SS $ as input for training. The Mahalanobis distance metric $\X$
        can be viewed as a linear transformation in the Euclidean space by
        projecting the data using matrix $\L$ $(\X = \L\L^\T)$. That is, nearest neighbors
        of samples using Mahalanobis distance metric $\X$ are the same as nearest
        neighbors using Euclidean distance in the transformed space. 
        \BoostMetric assumes that the triplets of input training set
        approximately represent the actual nearest neighbors of samples in the transformed
        space defined by the Mahalanobis metric. However, even though the
        triplets of \BoostMetric consist of nearest neighbors of the original training
        samples, generated triplets are not exactly the same as the actual
        nearest neighbors of training samples in the transformed space by $\L$.

        We can refine the results of \BoostMetric iteratively, as in the
        multiple-pass LMNN \cite{Weinberger2009Distance}:
        \BoostMetric can estimate the triplets in
        the transformed space under a multiple-pass procedure as close to actual
        triplets as possible. The rule for multi-pass \BoostMetric is
        simple. At each pass $p$ ($p=1,2,\cdots)$,
        we decompose the learned Mahalanobis distance
        metric $\X_{p-1}$ of previous pass into transformation matrix $\L_p$.
        The initial matrix $\L_1$ is an identity matrix. 
        Then we generate the
        training set triplets from the set of points $\{\L^\T \ba_1, \dots,
        \L^\T \ba_m\}$ where $\L = \L_1 \cdot \L_2 \cdot \cdots \L_p$. The
        final Mahalanobis distance metric $\X$ becomes $\L\L^\T$ in
        Multi-pass \BoostMetric.
\begin{sidewaystable}%
\begin{center}
\caption{Comparison of test classification error rates (\%) of a
$3$-nearest neighbor classifier on benchmark datasets. Results of
NCA are not available either because the algorithm does not converge
or due to the out-of-memory problem. BoostMetric-E indicates
\BoostMetric with the exponential loss and BoostMetric-L is \BoostMetric
with the logistic loss; both use stage-wise optimization. 
``MP'' means Multiple-Pass \BoostMetric and ``TC'' is
\BoostMetric with totally corrective optimization. We report
computational time as well.
 }\label{Table:ExperimentalResults}
\centering \small
\begin{tabular}{l|l||c|c|c|c|c|c|c}
\hline
\multicolumn{2}{l||}{}                 & MNIST & USPS & Letters & yFaces &  bal & wine & iris\\
\hline\hline

\multicolumn{2}{l||}{ \# of {samples} }   & 70,000 & 11,000 & 20,000 &
2,414 & 625  & 178 & 150\\\hline

\multicolumn{2}{l||}{ \# of {triplets}}  & 450,000 & 69,300 & 94,500 &
15,210 & 3,942& 1,125& 945\\\hline

\multicolumn{2}{l||}{ dimension  }  & 784    & 256    & 16     &
1,024 & 4 & 13 & 4
\\\hline

\multicolumn{2}{l||}{ dimension after PCA  } & 164    & 60     &        &
300   &  & &
\\\hline

\multicolumn{2}{l||}{ \# of samples for training } & 50,000 & 7,700  & 10,500 & 1,690 &
438  & 125 & 105\\\hline

\multicolumn{2}{l||}{\# cross validation samples}  & 10,000  & 1,650  & 4,500  & 362 &
94   & 27 & 23\\\hline

\multicolumn{2}{l||}{\# test samples} & 10,000  & 1,650  & 5,000  & 362 &
93   & 26 & 22\\\hline

\multicolumn{2}{l||}{\# of classes} & 10     & 10     & 26     & 38
&  3    & 3 & 3\\\hline

\multicolumn{2}{l||}{\# of runs}  & 1      & 10      & 1     & 10    & 10    &  10  & 10 \\

\hline \hline

\multirow{15}{*}{\begin{sideways}\parbox{30mm}{\centering
\textbf{Error Rates}}\end{sideways}}
 & Euclidean & 3.19 &  4.78 (0.40) &  5.42 &  28.07 (2.07) &  18.60 (3.96) & 28.08 (7.49) & 3.64 (4.18)\\
\cline{2-9}

& PCA & 3.10 &  3.49 (0.62) &  &  28.65 (2.18) &  & & \\ \cline{2-9}

& LDA & 8.76 & 6.96 (0.68) & 4.44 &  \textbf{5.08 (1.15)} &  12.58 (2.38) & 0.77 (1.62) & 3.18 (3.07)\\
\cline{2-9}

& RCA & 7.85 & 5.35 (0.52) & 4.64 &  7.65 (1.08) &  17.42 (3.58) & \textbf{0.38 (1.22)} & 3.18 (3.07)\\
\cline{2-9}

& NCA &  &  &  &   &  18.28 (3.58) & 28.08 (7.49) & 3.18 (3.74)\\
\cline{2-9}

& LMNN & 2.30 &  3.49 (0.62) &  3.82 &  14.75 (12.11) & 12.04 (5.59)
& 3.46 (3.82) & 3.64 (2.87)\\ \cline{2-9}

& ITML &   2.80 &  3.85 (1.13) &  7.20 &  19.39 (2.11) & 10.11
(4.06) & 28.46 (8.35) & 3.64 (3.59)
\\ \cline{2-9}

& BoostMetric-E & 2.65 & 2.53 (0.47) &
3.06 &  6.91 (1.90) &  10.11 (3.45) & 3.08 (3.53) & 3.18 (3.74) \\
\cline{2-9}

& BoostMetric-E, MP & 2.62    & 2.24 (0.40) & 2.80 & 6.77 (1.77) &
10.22 (4.43) & 1.92 (2.03) & 3.18 (4.31)\\\cline{2-9}

& BoostMetric-E, TC & 2.20 & 2.25 (0.51) & 2.82 & 7.13 (1.40) &
10.22 (2.39) & 4.23 (3.82) & 3.18 (3.07)
\\\cline{2-9}

& BoostMetric-E, MP, TC & 2.34 &  2.23 (0.34) & 3.74 &
7.29 (1.58)  & 10.32 (3.09)  & 2.69 (3.17)  & 3.18 (4.31)\\
\cline{2-9}

& BoostMetric-L & 2.66 &  2.38 (0.31) &
2.80 & 6.93 (1.59) & 9.89 (3.12) & 3.08 (3.03) & 3.18 (3.74) \\
\cline{2-9}

& BoostMetric-L, MP & 2.72 & 2.22 (0.31) & 2.70 & 6.66 (1.35) &
10.22 (4.25)    & 1.15  (1.86)    & 3.18 (4.31)\\ \cline{2-9}

& BoostMetric-L, TC & \textbf{2.10} &  \textbf{2.13 (0.41)} &
2.48 & 7.71 (1.68) & 9.57 (3.18) & 3.85 (4.05) & 3.64 (2.87) \\
\cline{2-9}

& BoostMetric-L, MP, TC & 2.11     & 2.10 (0.42)  &
\textbf{2.36}  & 7.15 (1.32)  & \textbf{8.49 (3.71)}   & 3.08 (3.03)  & \textbf{2.73 (2.35)}\\
\hline \hline

\multirow{4}{*}{\begin{sideways}\parbox{10mm}{\textbf{Comp.
Time}}\end{sideways}}
& LMNN & 10.98h &  20s &  1249s &  896s &  5s & 2s & 2s\\
\cline{2-9}

& ITML &  0.41h & 72s  &  55s &  5970s &  8s & 4s & 4s\\
\cline{2-9}

& BoostMetric-E & 2.83h & 144s & 3s &  628s &  less than 1s & 2s & less than 1s\\
\cline{2-9}

& BoostMetric-L & 0.89h &  65s &  34s &  256s &  less than 1s & 2s & less than 1s\\
\hline

\end{tabular}
\end{center}
\end{sidewaystable}

\section{Experiments}
\label{sec:exp}

        In this section, we present experiments on data visualization,
        classification and image retrieval tasks.

\subsection{An Illustrative Example}

\begin{figure*}[t]
    \centering
    \fbox{
    \includegraphics[width=0.25\textwidth,height=0.22\textwidth]
                    {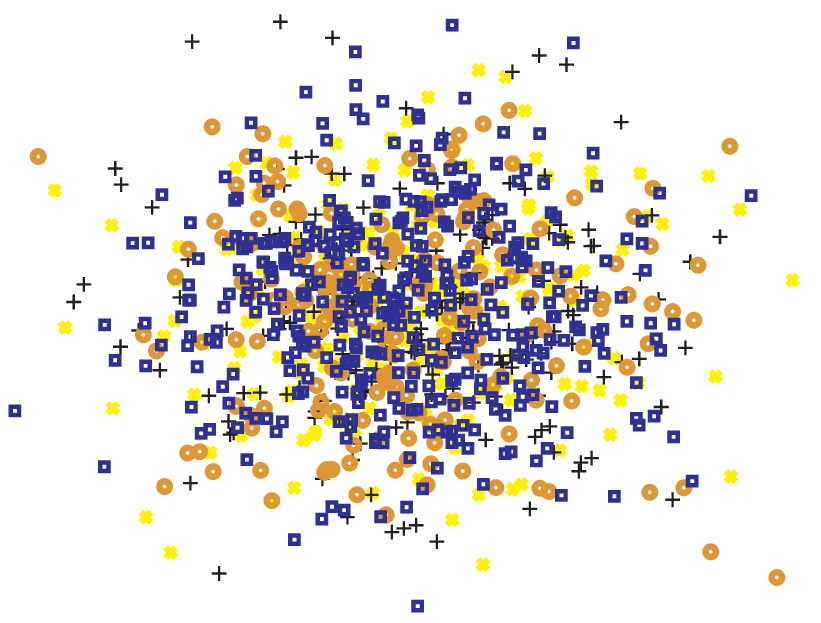}
    }
    \fbox{
    \includegraphics[width=0.25\textwidth,height=0.22\textwidth]
                     {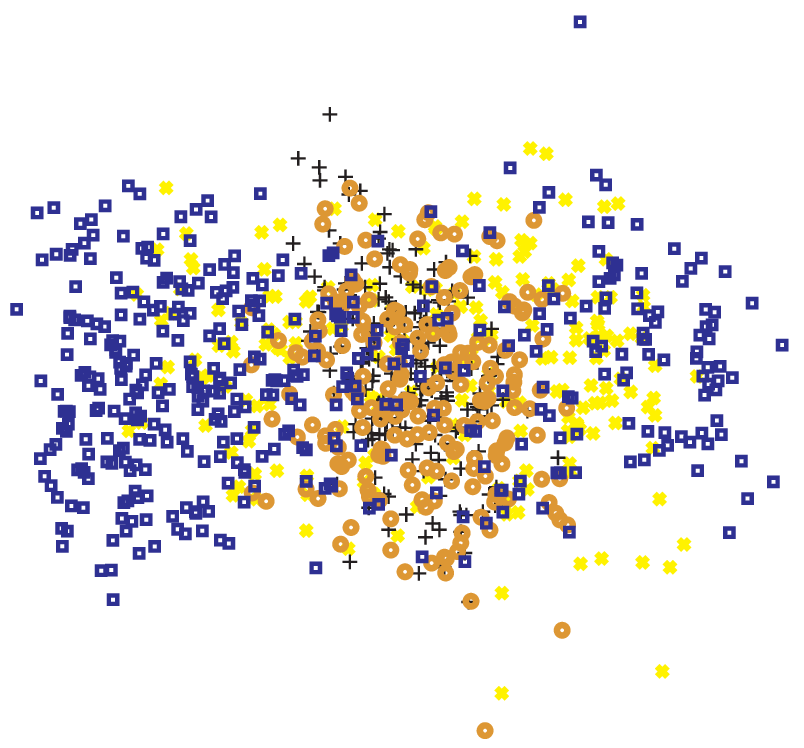}
    }
    \fbox{
    \includegraphics[width=0.25\textwidth,height=0.22\textwidth]
                      {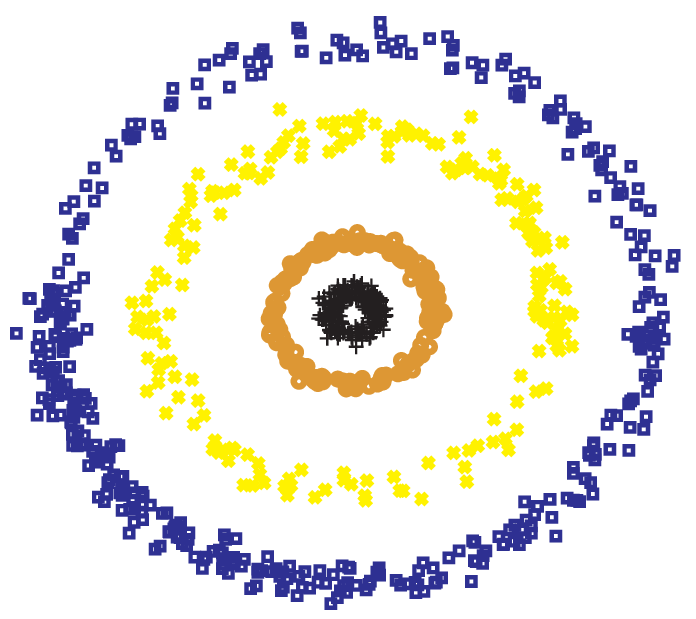}
    }
    \caption{The data are projected into 2D with
    PCA (left), LDA (middle) and \BoostMetric (right).
    Both PCA and LDA fail to recover the data structure.
    The local structure of the data is
    preserved after projection by \BoostMetric.}
    \label{fig:circle}
\end{figure*}

        We demonstrate a data visualization problem on an artificial
        toy dataset (concentric circles) in Fig.~\ref{fig:circle}.
        The dataset has four
        classes. The first two dimensions follow concentric circles
        while the left eight dimensions are all random Gaussian noise.
        In this experiment, $ 9000$ triplets are generated for
        training.
        When the scale of the noise is large, PCA fails find the
        first two informative dimensions.
        LDA fails too because clearly each class does not follow a
        Gaussian distraction and their centers overlap at the same
        point. The proposed \BoostMetric algorithm find the informative
        features. The eigenvalues of $ \X$ learned by \BoostMetric
        are $\{ 0.542, 0.414, 0.007, 0, \cdots, 0  \} $, which
        indicates that \BoostMetric successfully reveals the data's
        underlying 2D structure.
        We have used the exponential loss in this experiment.

\subsection{Classification on Benchmark Datasets}

        We evaluate \BoostMetric on $ 7 $ datasets of different sizes.
        Some of the datasets have very high dimensional inputs. We use PCA
        to decrease the dimensionality before training on these datasets
        (MNIST, USPS and yFaces). PCA pre-processing helps to eliminate noises and
        speed up computation. Table~\ref{Table:ExperimentalResults} summarizes the datasets
        in detail.
        We have used USPS and MNIST handwritten digits,
        Yale face recognition datasets,
        and a few UCI machine learning
        datasets\footnote{\url{http://archive.ics.uci.edu/ml/}}
        .

        Experimental results are obtained by averaging over $ 10 $ runs (except
        for large datasets MNIST and Letter).
        We randomly split the datasets for each run.
        We have used the same mechanism to generate training triplets as described in
        \cite{Weinberger05Distance}. Briefly, for each training point $ \ba_i $,
        $ k $ nearest neighbors that have same labels as $ y_i $ (targets), as well as
        $ k $ nearest neighbors that have different labels from $ y_i $ (imposers)
        are found.
        We then construct triplets from $ \ba_i $ and  its corresponding targets and imposers.
        For all the datasets, we have set $ k = 3 $ ($3$-nearest-neighbor).
        We have compared our method against a few methods:
        RCA \cite{Bar2005Mahalanobis}, NCA \cite{Goldberger2004Neighbourhood}, ITML~\cite{Davis2007Info} and LMNN
        \cite{Weinberger05Distance}. Also in Table~\ref{Table:ExperimentalResults},
        ``Euclidean'' is the baseline algorithm that uses the standard Euclidean distance.
        The codes for these compared algorithms are downloaded from the corresponding author's
        website.
        Experiment setting for LMNN follows
        \cite{Weinberger05Distance}. The slack variable parameter
        for ITML is tuned using cross validation over the values
        ${0.01, 0.1,
        1, 10}$ as in ~\cite{Davis2007Info}.
        For \BoostMetric, we have set $ v = 10^{-7} $, 
        the maximum number of  iterations $ J = 500 $.

        \BoostMetric has different variants which use
        1) the exponential loss (\BoostMetric-E), 
        2) the logistic loss (\BoostMetric-L), 
        3) multiple pass evaluation (MP) for updating triplets
        with the exponential and logistic loss, and 
        4) two optimization strategies, namely, stage-wise optimization and totally corrective
        optimization.  The experiments are conducted by using Matlab
        and a C-mex implementation of the L-BFGS-B algorithm.

        As reported in Table \ref{Table:ExperimentalResults}, we
        can conclude: 1) \BoostMetric consistently improves the
        accuracy of $ k $NN classification using Euclidean distance on
        most datasets. So learning a Mahalanobis metric based upon the
        large margin concept indeed leads to improvements in $ k$NN
        classification.  2) \BoostMetric outperforms other
        state-of-the-art algorithms in most cases (on $5$ out of $7$
        datasets).  LMNN is the second best algorithm on these $ 7 $
        data sets statistically.  LMNN's results are consistent with
        those given in \cite{Weinberger05Distance}.  ITML is faster
        than \BoostMetric on most large datasets such as MNIST.
        However it has higher error rates than \BoostMetric in our
        experiment.  3) NCA can only be run on a few small data sets.
        In general NCA does not perform well. Initialization is
        important for NCA because NCA's objective function is highly
        non-convex and can only find a local optimum.

        In this experiment, LMNN solves for the global optimum
        (learning $ \X$) except for the Wine dataset. When the LMNN
        solver solves for $ \X $ on the Wine dataset, the error rate
        is large ($20.77\% \pm 14.18\%$). So instead we have solved
        for the projection matrix $ \L $ on Wine. Also note that the
        number of training data on Iris, Wine and Bal in
        \cite{Weinberger05Distance} are different from our experiment.
        We have used these datasets from UCI.
        For the experiment on MNIST, if we deskew the handwritten
        digits data first as in \cite{Weinberger2009Distance},
        the final accuracy can be slightly
        improved. Here we have not deskewed the data.

\begin{table*}[t!]
\centering
{
\begin{tabular}{l|c|c|c|c|c}
\hline
    $ v $    & $  10^{-8} $ & $ 10^{-7} $  & $  10^{-6} $  & $  10^{-5} $   & $ 10^{-4} $ 
    \\ \hline\hline
    Bal        & 8.98 (2.59)  &  8.88 (2.52) & 8.88 (2.52)  & 8.88 (2.52) & 8.93 (2.52)  \\
    B-Cancer   & 2.11 (0.69)  &  2.11 (0.69) & 2.11 (0.69)  & 2.11 (0.69)  & 2.11 (0.69)  \\ 
    Diabetes   & 26.0 (1.33)  &  26.0 (1.33) & 26.0 (1.33)  & 26.0 (1.34) & 26.0 (1.46)  \\
\hline
\end{tabular}
}
\caption{Test error (\%) of a $ 3 $-nearest neighbor classifier 
with different 
values of the parameter $ v $.
Each experiment is run $10$ times.
We report the mean and variance. As expected,  as long as $ v $ is sufficiently small, 
in a wide range it almost does not
affect the final classification performance.  
}
\label{Table:v}
\end{table*}

\subsubsection{Influence of $ v $}

        Previously, we claim that the stage-wise version of
        \BoostMetric
        is parameter-free like AdaBoost.
        However, we do have a parameter $ v $. Actually,
        AdaBoost simply set $ v = 0 $.
        The coordinate-wise gradient descent optimization strategy of AdaBoost
        leads to an $ \ell_1$-norm  regularized maximum margin classifier
        \cite{Rosset2004Boosting}. It is shown that
        AdaBoost minimizes its loss criterion
        with an $ \ell_1 $ constraint on the coefficient vector.
        Given the similarity of the optimization of \BoostMetric with AdaBoost,
        we conjecture that \BoostMetric has the same property. Here we empirically
        prove that {\em as long as $ v $ is sufficiently small, the final performance
        is not affected by the value of $ v $}.
        We have set $ v $ from $ 10^{-8} $ to $ 10^{-4} $ and run \BoostMetric
        on $ 3 $ UCI datasets.
        Table~\ref{Table:v} reports the final $ 3$NN  classification error
        with different $ v $.
        The results are nearly identical.
        
        For the totally corrective version of \BoostMetric, similar results are 
        observed. Actually for LMNN, it was also reported that the regularization
        parameter does not have a significant impact on the final results
        in a wide range \cite{Weinberger2009Distance}.

\subsubsection{Computational time}

        As we discussed, one major issue
        in learning a Mahalanobis distance is heavy computational cost
        because of the semidefiniteness constraint.

        We have shown the running time of the proposed algorithm in Table~\ref{Table:ExperimentalResults}
        for the classification tasks\footnote{We
        have run all the experiments on a desktop with
        an Intel Core$^{\rm TM}$2 Duo CPU, $ 4 $G
        RAM and Matlab 7.7 (64-bit version).}.
        Our algorithm is generally fast. Our algorithm involves matrix operations and
        an EVD for finding its largest eigenvalue and its corresponding eigenvector.
        The time complexity of this EVD is $ O(D^2) $ with $ D $ the input dimensions.
         We compare our algorithm's running time with LMNN in Fig.~\ref{fig:cputime} on the
         artificial dataset (concentric circles). 
         Our algorithm is stage-wise \BoostMetric with the exponential loss. 
         We vary the input dimensions from
         $ 50 $ to $ 1000 $ and keep the number of triplets fixed to $ 250 $.
        LMNN does not use standard interior-point
        SDP solvers, which do not scale well. Instead LMNN heuristically combines sub-gradient
        descent in both the matrices $ \L$  and $ \X $.
        At each iteration, $ \X $ is projected back onto the \PSD cone using EVD. So a full EVD
        with time complexity $ O(D^3)$ is needed. Note that LMNN is much faster than SDP solvers
        like CSDP \cite{Borchers1999CSDP}. As seen from Fig.~\ref{fig:cputime}, when the input
        dimensions are low, \BoostMetric is comparable to LMNN. As expected, when the input
        dimensions become large, \BoostMetric is significantly faster than LMNN.  Note that our
        implementation is in Matlab. Improvements are expected if implemented in C/C++.

\begin{figure}[t!]
    \centering
    \includegraphics[width=0.5\textwidth]
                    {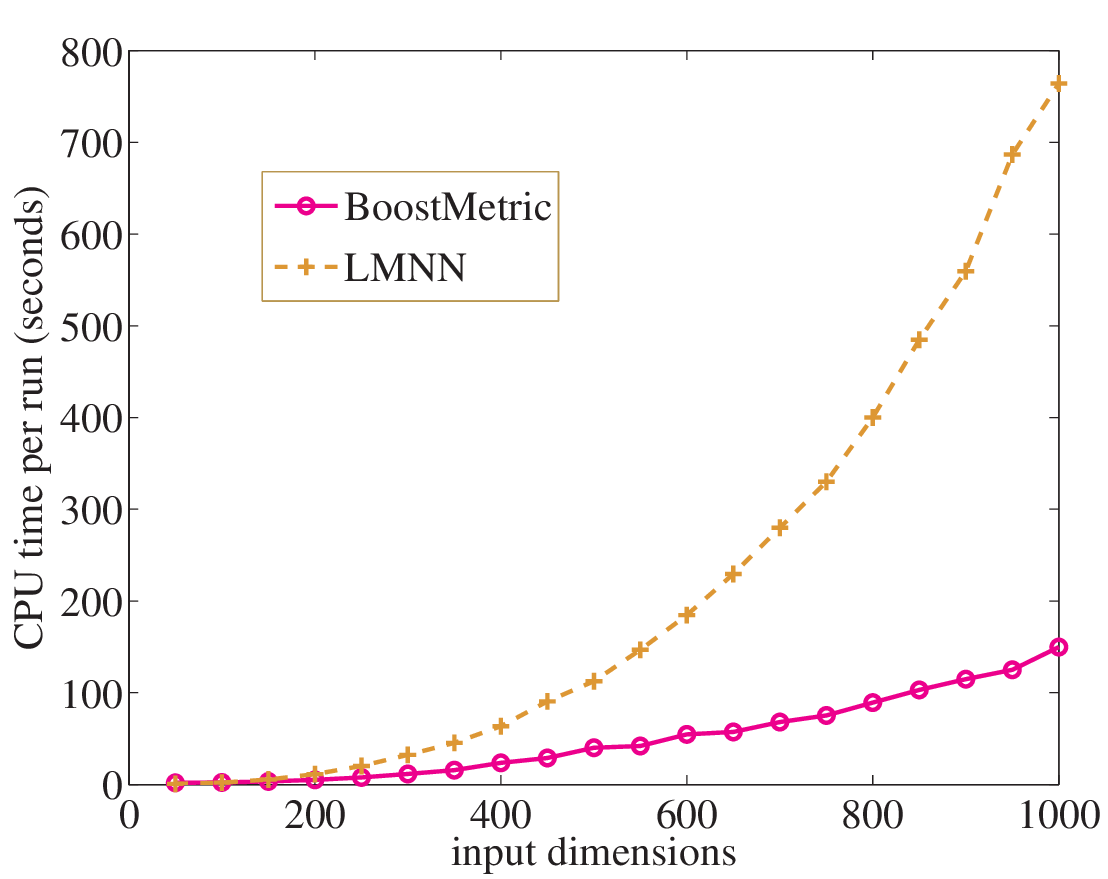}
    \caption{ 
    Computation time of the proposed \BoostMetric (stage-wise, exponential loss) 
    and the LMNN method \vs the input data's
    dimensions on an artificial dataset.  \BoostMetric is faster than LMNN with large input
    dimensions because at each iteration \BoostMetric only needs to calculate the largest
    eigenvector and LMNN needs a full eigen-decomposition. 
    }
    \label{fig:cputime}
\end{figure}

\subsection{Visual Object Categorization}
        
      In the following experiments, unless otherwise specified,   
      \BoostMetric means the stage-wise \BoostMetric with the exponential loss. 

      The proposed \BoostMetric and the LMNN are further compared on visual object categorization tasks. 
      The first experiment uses four classes of the Caltech-101 object recognition database \cite{Feifei2006Oneshot},
      including Motorbikes ($798$ images),
      Airplanes ($800$), Faces ($435$), and Background-Google
      ($520$). The task is to label each image according to the presence of a
      particular object. This experiment involves both object categorization
      (Motorbikes \vs Airplanes)
      and object retrieval (Faces \vs Background-Google)
      problems. In the second experiment, we compare the two methods on the MSRC
      data set including $240$ images\footnote{See 
\url{http://research. microsoft. com / en-us / projects / objectclassrecognition/}.
}.
    The objects in the images can be categorized into nine classes, including
    \textit{building, grass, tree, cow, sky, airplane, face, car and bicycle}.
    Different from the first experiment, each image in this database often
    contains multiple objects. The regions corresponding to each object have
    been manually pre-segmented, and the task is to label each region according
    to the presence of a particular object. Some examples are shown in
    Fig.~\ref{fig:objcate-examples}.

\begin{figure*}[t!]
\centering
\begin{tabular}{cccc}
\includegraphics[height=1.97698765cm]{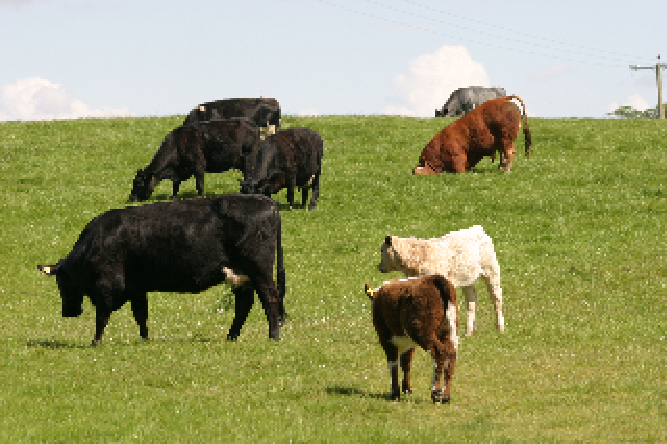} 
\includegraphics[height=1.97698765cm]{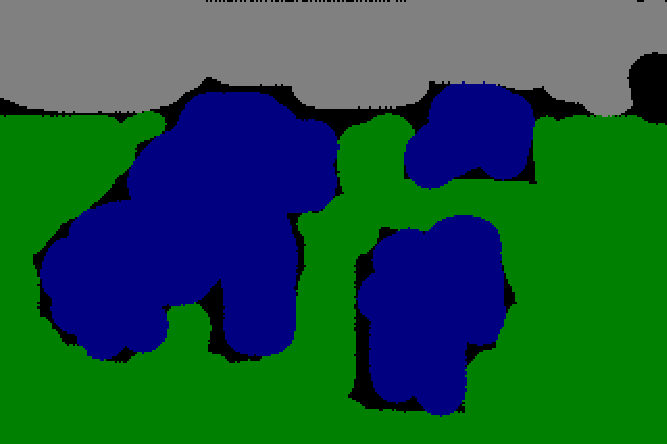}
\includegraphics[height=1.97698765cm]{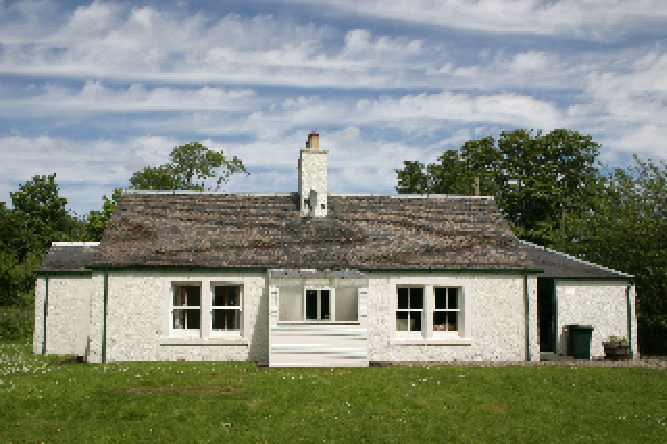}
\includegraphics[height=1.97698765cm]{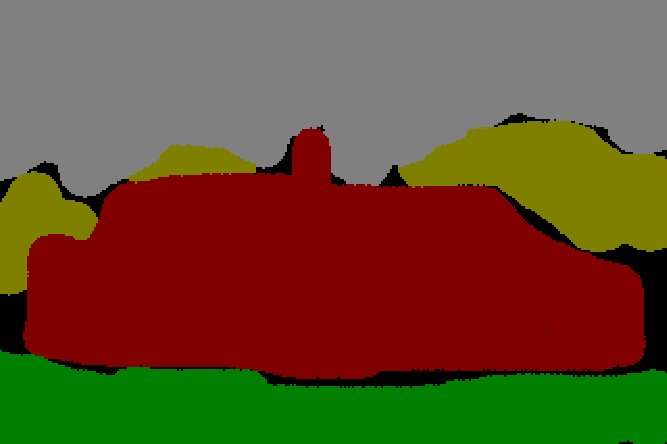}\\
\includegraphics[height=1.97698765cm]{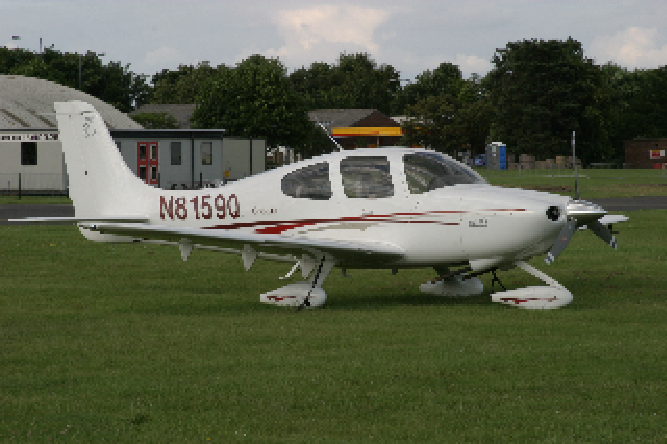}
\includegraphics[height=1.97698765cm]{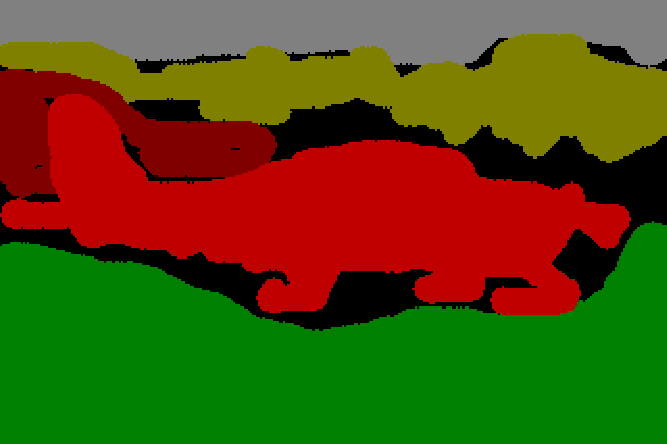}
\includegraphics[height=1.97698765cm]{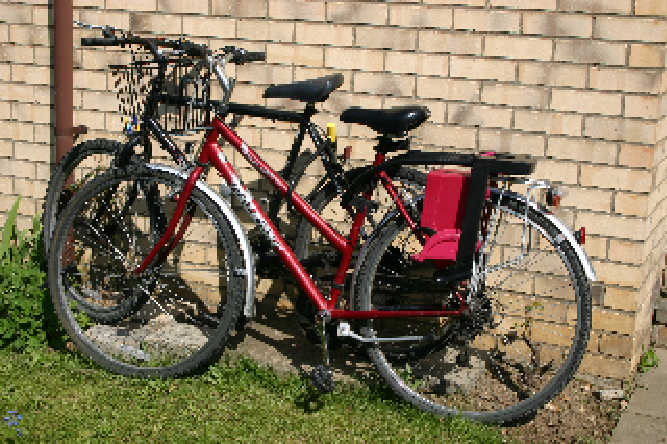}
\includegraphics[height=1.97698765cm]{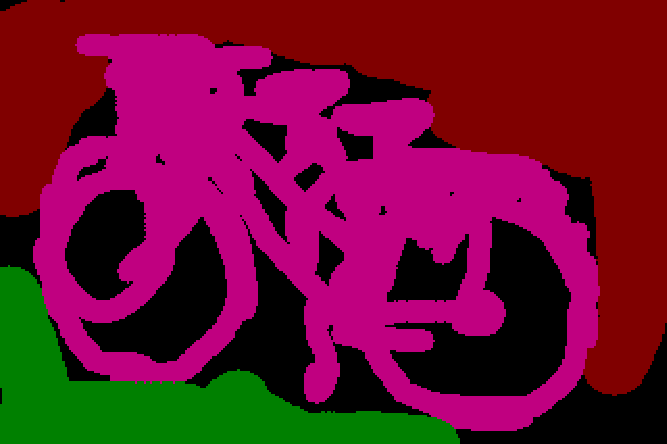}
\end{tabular}
\caption{Examples of the images in the MSRC data set and the pre-segmented regions labeled using different colors.}
\label{fig:objcate-examples}
\end{figure*}

\subsubsection{Experiment on the Caltech-101 dataset}

      For each image of the four classes, a number of interest regions are identified by
      the Harris-affine detector \cite{Mikolajczyk2004Scale}
      and each region is characterized by the SIFT
      descriptor \cite{lowe2004sift}. The total number of interest regions 
      extracted from the four classes are
      about $134,000$, $84,000$, $57,000$, and $293,000$, respectively.
      To accumulate statistics, the images of two involved
      object classes are randomly split as $10$ pairs of training/test
      subsets. Restricted to the images in a training subset (those in
      a test subset are only used for test), their local descriptors
      are clustered to form visual words by using $k$-means
      clustering. Each image is then represented by a histogram
      containing the number of occurrences of each visual word.

      \textit{Motorbikes versus Airplanes} This experiment discriminates the images of a motorbike from
      those of an airplane. In each of the $10$ pairs of training/test
      subsets, there are $959$ training images and $639$ test images.
      Two visual codebooks of size $100$ and $200$ are
      used, respectively. With the resulting histograms, the
proposed \BoostMetric and the LMNN are learned on a training subset and evaluated on the corresponding
test subset. Their averaged classification error rates are compared in
Fig.~\ref{fig:motor} (left).
For both visual codebooks, the proposed \BoostMetric achieves lower error rates than
the LMNN and the Euclidean distance, demonstrating its superior
performance. We also apply a linear SVM classifier with its regularization
parameter carefully tuned by $5$-fold cross-validation.
Its error rates are $3.87\%\pm 0.69\%$ and $ 3.00\% \pm 0.72\% $ on
the two visual codebooks, respectively.
In contrast, a $3$NN with \BoostMetric has error rates
$ 3.63  \% \pm  0.68 \%$
and
$2.96 \% \pm 0.59\%$.
Hence, the performance of the proposed \BoostMetric
is comparable to the state-of-the-art SVM classifier.  Also, Fig.~\ref{fig:motor}
(right) plots
the test error of the \BoostMetric against
the number of triplets for training. The general trend is that more triplets lead to
smaller errors.

\begin{figure}[t!]
    \centering
    {
    \includegraphics[width=0.43\textwidth]
                    {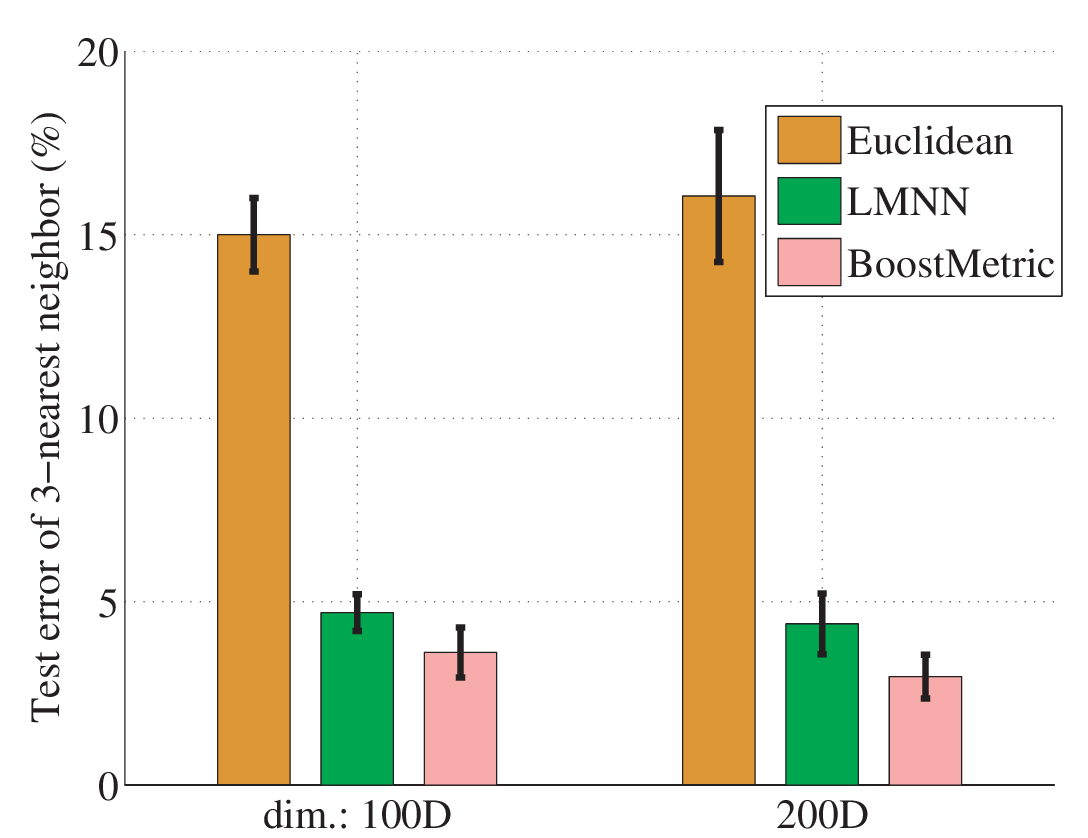}  
    \includegraphics[width=0.43\textwidth]
                    {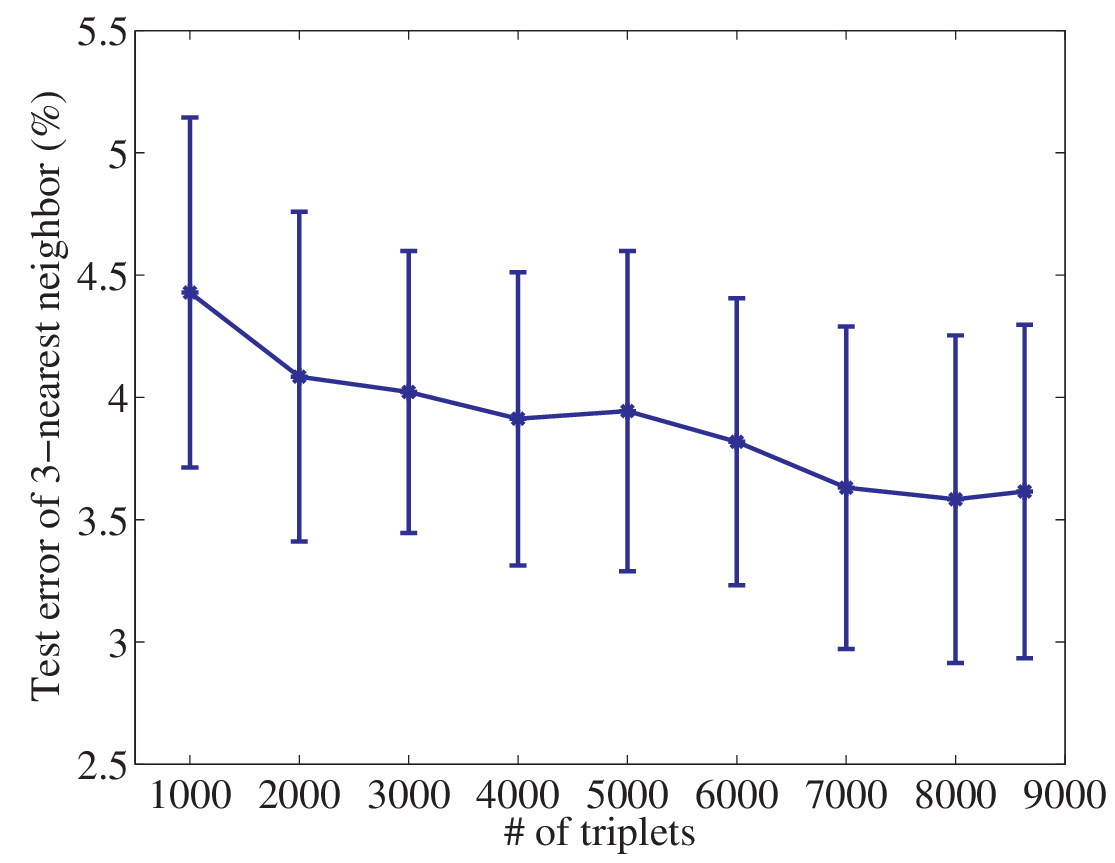}
    }
    \caption{ 
    Test error ($3$-nearest neighbor) of \BoostMetric on the Motorbikes \vs
    Airplanes 
    datasets. The second plot shows the test error against the
    number of training triplets with a 
    $100$-word codebook.
    }
    \label{fig:motor}
\end{figure}

       \textit{Faces versus Background-Google} This experiment uses
       the two object classes as a retrieval problem.  The target of
       retrieval is  face images. The images in the class of
       Background-Google are randomly collected from the Internet and
       they represent the non-target class.  \BoostMetric is first
       learned from a training subset and retrieval is conducted on
       the corresponding test subset. In each of the $10$
       training/test subsets, there are $573$ training images and
       $382$ test images. Again, two visual codebooks of size $100$
       and $200$ are used. Each face image in a test subset is used as
       a query, and its distances from other test images are
       calculated by the proposed BoostMetric,  LMNN and the Euclidean
       distance, respectively. For each metric, the \textit{Precision}
       of the retrieved top $5$, $10$, $15$ and $20$ images are
       computed. The \textit{Precision} values from each query are
       averaged on this test subset and then averaged over the 
       $10$ test subsets.  The retrieval precision of these metrics
       is  shown in Fig.~\ref{fig:face} (with a codebook size $ 100$).
       As we can see that the \BoostMetric consistently attains the highest
       values on both visual codebooks, which again verifies its
       advantages over  LMNN and Euclidean distance. With a
       codebook size $ 200$, very similar results are obtained.

\subsubsection{Experiment on the MSRC dataset}

            The $240$ images of the MSRC database are randomly halved
            into $10$ groups of training and test sets. Given a set of
            training images, the task is to predict the class label
            for each of the pre-segmented regions in a test image. We
            follow the work in ~\cite{Winn05objectcategorization} to
            extract features and conduct experiments.   Specifically,
            each image is converted from the RGB color space to the
            CIE Lab color space. First, three Gaussian low-pass
            filters are applied to the $L$, $a$, and $b$ channels,
            respectively.  The standard deviation $\sigma$ of the
            filters are set to $1$, $2$, and $4$, respectively, and
            the filter size is defined as $4\sigma$.  This step
            produces $9$ filter responses for each pixel in an image.
            Second, three Laplacian of Gaussian~(LoG) filters are
            applied to the $L$ channel only, with $\sigma = 1, 2, 4,
            8$ and the filter size of $4\sigma$. This step gives rise
            to $4$ filter responses for each pixel. Lastly, the first
            derivatives of the Gaussian filter with $\sigma = 2, 4$
            are computed from the $L$ channel along the row and column
            directions, respectively. This results in $4$ more filter
            responses. After applying this set of filter banks, each
            pixel is represented by a $17$-dimensional feature
            vectors. All the feature vectors from a training set are
            clustered using the $k$-means clustering with a
            Mahalanobis distance\footnote{Note that this Mahalanobis
            distance is different from the one that we are going to
            learn with the \BoostMetric.}. By setting $k$ to $2000$, a
            visual codebook of $2000$ visual words is obtained. We
            implement the word-merging approach in
            ~\cite{Winn05objectcategorization} and obtain a compact
            and discriminative codebook of $300$ visual words. Each
            pre-segmented object region is then represented as a
            $300$-dimensional histogram.

\begin{figure}[t!]
    \centering
    \includegraphics[width=0.5\textwidth]
                    {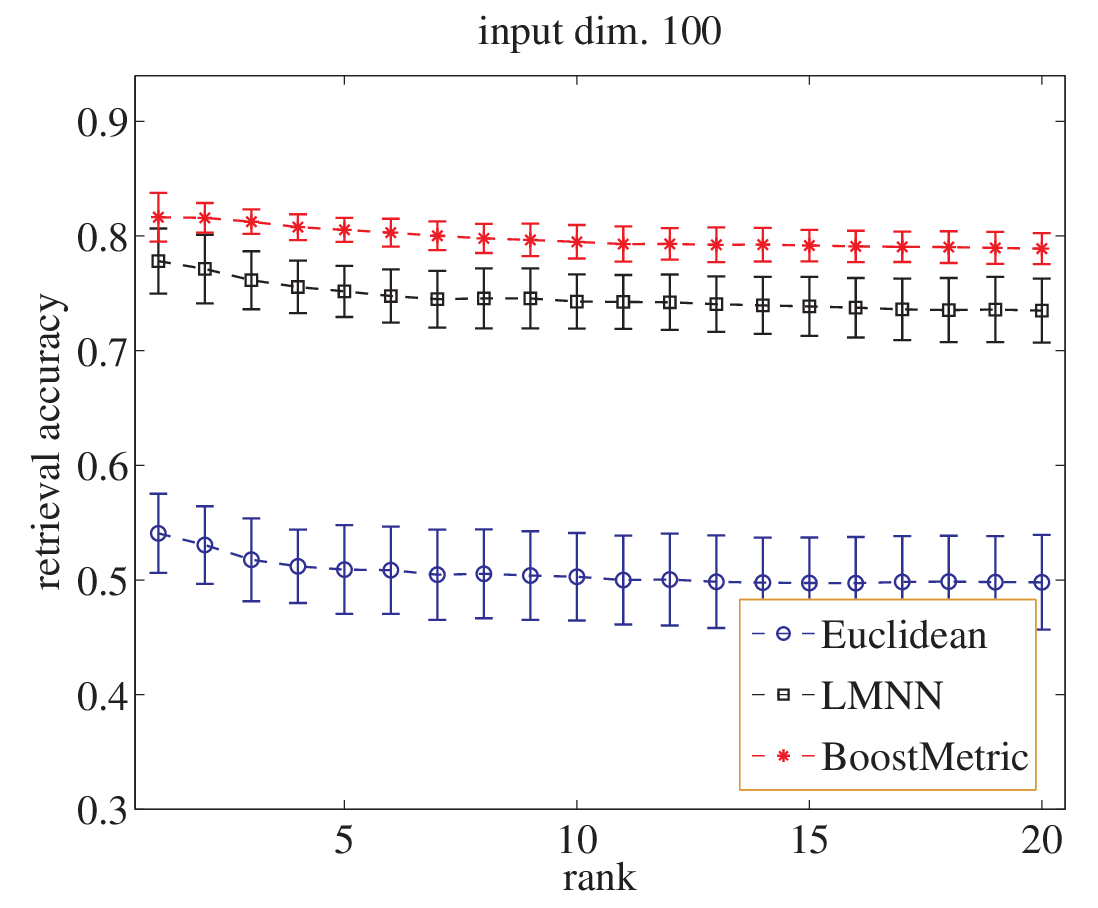}
     \caption{ 
    Retrieval accuracy of distance metric learning algorithms
    on the Faces \vs  B\-a\-ck\-gr\-ound-G\-oo\-g\-le dataset.
    Error bars show the standard
    deviation. 
    }
    \label{fig:face}
\end{figure}

The proposed \BoostMetric is compared with the LMNN algorithm as follows.
With $10$ nearest neighbors information, about $20,000$ triplets are
constructed and used to train the \BoostMetric. To ensure
convergence, the maximum number of iterations is set as $5000$ in
the optimization of training \BoostMetric. The training of  LMNN
 follows the default setting.  $k$NN
classifiers with the two learned Mahalanobis distances and the
Euclidean distance are applied to each training and test group to
categorize an object region. The categorization error rate on each
test group is summarized in Table~\ref{Table:objcate-acc}. As
expected, both learned Mahalanobis distances achieve superior
categorization performance to the Euclidean distance. Moreover, the
proposed \BoostMetric achieves better performance than the LMNN, as
indicated by its lower average categorization error rate and the
smaller standard deviation. Also, the $k$NN classifier using the
proposed \BoostMetric achieves comparable or even higher
categorization performance than those reported in
~\cite{Winn05objectcategorization}. Besides the categorization
performance, we compare the computational efficiency of the
\BoostMetric and the LMNN in learning a Mahalanobis distance. The
computational time result is based on the Matlab codes for both methods.
    In this experiment, the average time cost by the \BoostMetric for learning the
    Mahalanobis distance is $3.98$ hours, whereas the LMNN takes
    about $8.06$ hours to complete this process. Hence, the proposed
    \BoostMetric has a shorter training process than the LMNN method.
    This again demonstrates the computational advantage of the
    \BoostMetric over the LMNN method.

\begin{table}
\centering
 {\small
\begin{tabular}{c|c|c|c}

  \hline
  group index  & Euclidean & LMNN & \BoostMetric \\
    \hline \hline
  1     &   9.19    &   6.71    &     {4.59}    \\
  2     &   5.78    &   3.97    &     {3.25}    \\
  3     &   6.69    &   2.97    &     {2.60}    \\
  4     &   5.54    &   {3.69}  &     4.43  \\
  5     &   6.52    &   5.80        &     {4.35}    \\
  6     &   7.30    &   4.01    &     {3.28}    \\
  7     &   7.75    &   {2.21}  &     2.58  \\
  8     &   7.20    &   {4.17}  &     4.55  \\
  9     &   6.13    &   {3.07}  &     4.21  \\
 10     &   8.42    &   {5.13}  &     5.86  \\
 \hline \hline
 average:   &   7.05    &   4.17    &   \textbf{3.97}   \\
 standard devision:   &   1.16    &   1.37    &   \textbf{1.03} \\
 \hline
\end{tabular}
}
\caption{Comparison of the categorization performance.} 
\label{Table:objcate-acc}
\end{table}

\subsection{Unconstrained Face Recognition}

We  use the ``labeled faces in the wild'' (LFW)
dataset~\cite{LFWTech} for face recognition in this experiment.
\begin{figure}[t]
\centering
\begin{tabular}{ccc | ccc}
\includegraphics[height=1.9975cm]{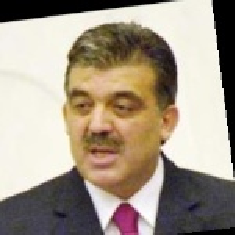}    &
\includegraphics[height=1.9975cm]{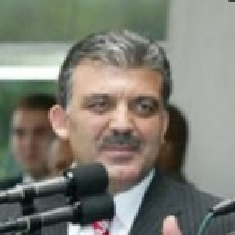}    &
\includegraphics[height=1.9975cm]{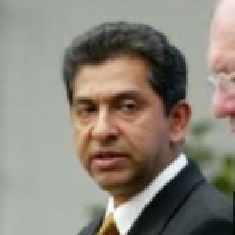}  &
\includegraphics[height=1.9975cm]{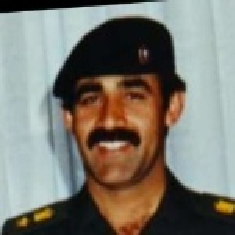}  &
\includegraphics[height=1.9975cm]{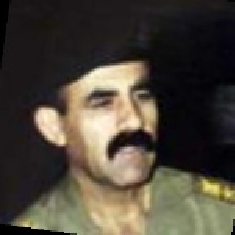}  &
\includegraphics[height=1.9975cm]{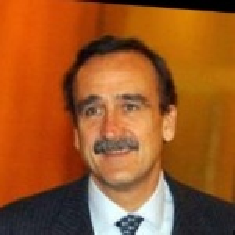} \\
\includegraphics[height=1.9975cm]{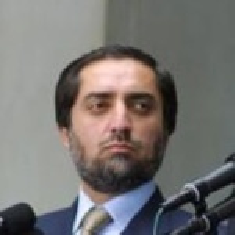}&
\includegraphics[height=1.9975cm]{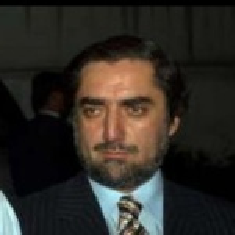}&
\includegraphics[height=1.9975cm]{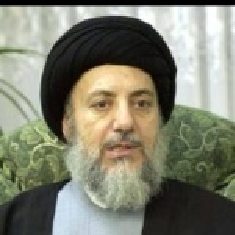}&
\includegraphics[height=1.9975cm]{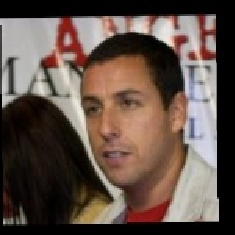}&
\includegraphics[height=1.9975cm]{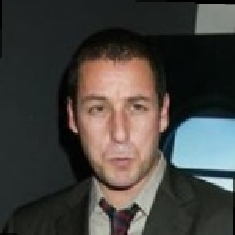}&
\includegraphics[height=1.9975cm]{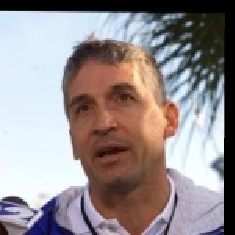}
\end{tabular}
\caption{Four generated triplets based on the pairwise
information provided by the LFW data set. For the three images in each
triplet,
the first two belong to the
same individual and the third one is a different individual.}
\label{fig:LFW_trip}
\end{figure}

This is a data set of unconstrained face images, which has a large
range of variations seen in real world, including $13,233$ images
of $5,749$ people collected from news articles on Internet.  The
face recognition task here is {\em pair matching}---given two face images,
to determine if these two images are of the same individual. So we
classify unseen pairs to determine whether each image in the pair indicates the same
individual or not, by applying M$k$NN of \cite{Guillaumin09isthat}
instead of $k$NN.

Features of face images are extracted by computing $3$-scale,
$128$-dimensional SIFT descriptors \cite{lowe2004sift}, which center
on $9$ points of facial features extracted by a facial feature
descriptor, same as described in \cite{Guillaumin09isthat}. PCA is
then performed on the SIFT vectors to reduce the dimension to
between $100$ and $400$.

\begin{table}[t]
\centering { \small
\begin{tabular}{ c ||c|c|c|c}
        \hline
 {number of triplets}  & 100D & 200D & 300D& 400D\\
\hline \hline
 3,000  & 80.91 (1.76) & 82.39 (1.73) & 83.40 (1.46) & 83.64 (1.66)\\ %

 6,000 & 81.13 (1.76) & 82.59 (1.84) & 83.58 (1.25) & 83.70 (1.73)\\ %

 9,000 & 81.01 (1.69) & 82.63 (1.68) & 83.65 (1.70) & 83.72 (1.47)\\ %

12,000 & 81.06 (1.63) & 83.00 (1.38) & 83.60 (1.89) & 83.57 (1.47)\\

15,000 & 81.10 (1.71) & 82.78 (1.83) & 83.69 (1.62) & 83.80 (1.85)\\

18,000 & 81.37 (2.15) & 83.19 (1.76) & 83.60 (1.66) & 83.81 (1.55)\\
\hline 
\end{tabular}
}
\caption{ Comparison of the face recognition accuracy (\%) of
our proposed \BoostMetric on the LFW dataset by varying the PCA dimensionality
and the number of triplets for each fold. }
\label{Table:lfw_error}
\end{table}

\begin{figure}[th]
\centering
\begin{tabular}{c}
\includegraphics[width=0.65\textwidth]{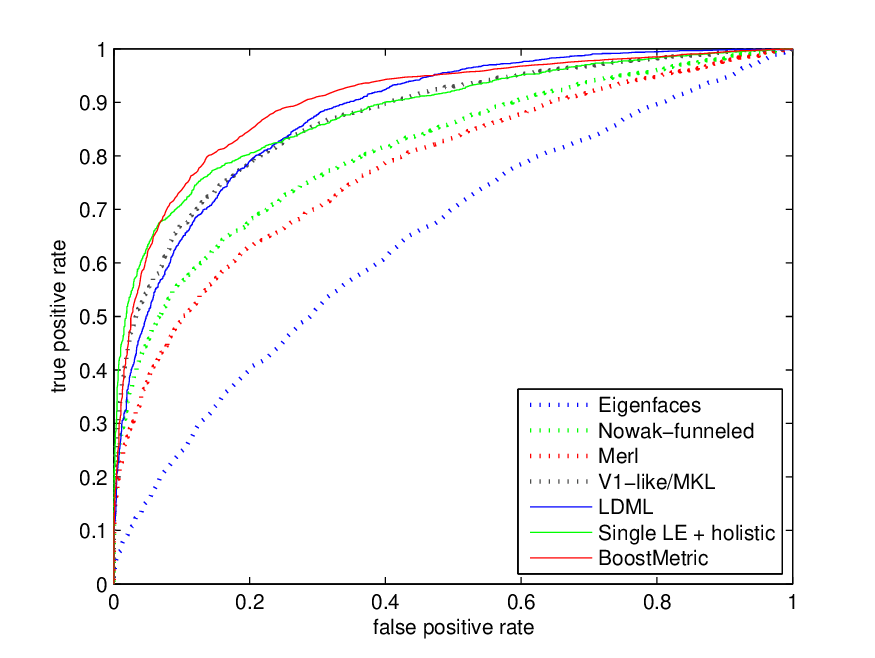}
\\
\includegraphics[width=0.65\textwidth]{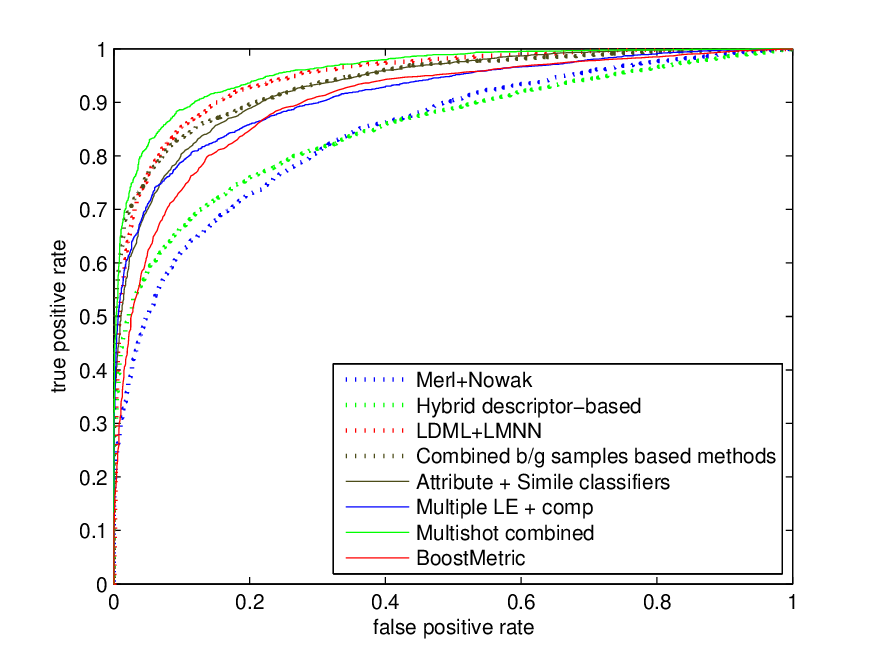}
\end{tabular}
\caption{ (top) ROC Curves that use a single descriptor and a single
classifier, (bottom) ROC curves that use hybrid descriptors are
plotted. Our \BoostMetric
with a single classifier is also plotted. 
Each point on the curves
is the average over the 10 folds of rates for a fixed threshold.}
\label{fig:LFW_ROC}
\end{figure}

\textit{Simple recognition systems with a single descriptor}
    Table~\ref{Table:lfw_error} shows our
    \BoostMetric's performance by varying PCA dimensionality and the
    number of triplets. Increasing the number of training triplets gives
    slight improvement of recognition accuracy.  The dimension after
    PCA has more impact on the final accuracy for this task.

    In Fig.~\ref{fig:LFW_ROC}, we have drawn ROC curves of other
    algorithms for face recognition. To obtain our ROC curve, M$k$NN has
    moved the threshold value across the distributions of match and
    mismatch similarity scores. Fig.~\ref{fig:LFW_ROC} (a) shows
    methods that use a single descriptor and a single classifier only. As
    can be seen, our system using \BoostMetric outperforms all the others
    in the literature with a very small computational cost.

\textit{Complex recognition systems with one or more descriptors}
    Fig.~\ref{fig:LFW_ROC} (b) plots the performance of more complicated
    recognition systems that use hybrid descriptors or combination of
    classifiers. See Table \ref{Table:lfw_error2} for details.
    We can see that the performance of our \BoostMetric is close 
    to the state-of-the-art.

    In particular, \BoostMetric outperforms the method of 
    \cite{Guillaumin09isthat}, which has a similar pipeline but uses
    LMNN for learning a metric. 
    This comparison also confirms the importance of learning an appropriate 
    metric for  vision problems.

\begin{table*}[tbh!]
\centering { \small%
\begin{tabular}{r||l|l}
        \hline
  & single descriptor + single classifier & multiple descriptors/classifiers\\
\hline \hline

\cite{Turk1991} & 60.02 (0.79) & - \\
 & {`Eigenfaces'} &  \\
\hline

\cite{Nowak2007} & 73.93 (0.49) & - \\
 & {`Nowak-funneled'} &  \\
\hline

\cite{Huang2008} & 70.52 (0.60) &  76.18 (0.58)\\
 & {`Merl'} & {`Merl+Nowak'} \\
\hline

\cite{Wolf2008} & - &  78.47 (0.51)\\
 &  & {`Hybrid descriptor-based'} \\
\hline

\cite{Wolf2009} & 72.02 & 86.83 (0.34) \\
& - & {`Combined b/g samples based'}\\
\hline

\cite{Pinto2009} & 79.35 (0.55) &  - \\
 & {`V1-like/MKL'} &  \\
\hline

\cite{Taigman2009} & 83.20 (0.77) & \textbf{89.50 (0.40)}\\
 & - & {`Multishot combined'}\\
\hline

\cite{Kumar2009} & - & 85.29 (1.23)\\
 & & {`attribute + simile classifiers'}\\
\hline

\cite{Cao2010} & 81.22 (0.53) & 84.45 (0.46) \\
 & {`single LE + holistic'} & {`multiple LE + comp'} \\
\hline

\cite{Guillaumin09isthat} & 83.2 (0.4) &  87.5 (0.4)\\
 & {`LDML'} & {`LMNN + LDML'}\\
\hline

\BoostMetric  & \textbf{83.81 (1.55)} & - \\
 & {`\BoostMetric' on SIFT } &  \\
\hline

\end{tabular}
}
\caption{Test accuracy in percentage (mean and standard deviation) on
the LFW dataset. ROC curve labels in
Fig.~\ref{fig:LFW_ROC} are described here with details.
}
\label{Table:lfw_error2}
\end{table*}

\section{Conclusion}

        \label{sec:conc}

        We have presented a new algorithm, \BoostMetric, to
        learn a positive semidefinite metric using boosting
        techniques. We have generalized AdaBoost in the sense
        that the weak learner of \BoostMetric is a matrix, rather
        than a classifier.
        Our algorithm is simple and efficient.
        Experiments show its better performance over
        a few state-of-the-art existing
        metric learning methods.
        We are currently combining the idea of on-line
        learning into \BoostMetric to make it handle even
        larger data sets.

        We also want to  
        learn a metric using \BoostMetric  in the semi-supervised, 
        and multi-task learning setting.
        It has been shown in \cite{Weinberger2009Distance} that 
        the classification performance can be improved by learning multiple 
        local metrics. We will extend \BoostMetric to learn multiple
        metrics. Finally, we will explore to generalize \BoostMetric for solving 
        more general semidefinite matrix learning problems in machine learning.

\end{document}